\documentclass{article}

\usepackage{cite}
\usepackage{amsmath,amssymb,amsfonts}
\usepackage[dvipdfmx]{graphicx}
\usepackage{multirow}
\usepackage{makecell}
\usepackage{amsthm}
\newtheorem{lemma}{Property}

\title{An extension of linear self-attention \\for in-context learning}

\author{Katsuyuki Hagiwara}

\date{\normalsize Faculty of Education, Mie University,\\ 1577
Kurima-Machiya-cho, Tsu, 514-8507, Japan\\
E-mail : hagi@edu.mie-u.ac.jp}

\def\R{\mathbb{R}}

\def\I{{\bf I}}
\def\O{{\bf O}}
\def\X{{\bf X}}
\def\A{{\bf A}}
\def\B{{\bf B}}

\def\P{{\bf P}}

\def\W{{\bf W}}
\def\C{{\bf C}}
\def\Z{{\bf Z}}
\def\V{{\bf V}}

\def\H{{\bf H}}

\def\Q{{\bf Q}}
\def\M{{\bf M}}

\def\Y{{\bf Y}}
\def\F{{\bf F}}

\def\b{{\boldsymbol{b}}}
\def\x{{\boldsymbol{x}}}
\def\u{{\boldsymbol{u}}}
\def\y{{\boldsymbol{y}}}

\def\w{{\boldsymbol{w}}}

\def\b{{\boldsymbol{b}}}

\def\z{{\boldsymbol{z}}}

\def\valpha{{\boldsymbol{\alpha}}}
\def\vtheta{{\boldsymbol{\theta}}}
\def\evw{\widehat{\boldsymbol{w}}}
\def\of{\overline{f}}

\def\invsqr{\sigma_{\rm invsqr}}
\def\idrelu{\sigma_{\rm id}}
\def\relu{\sigma_{\rm relu}}
\def\:{\hspace{0.4mm}{\rm :}\hspace{0.4mm}}
\def\FE{{\rm FE}}
\def\SUM{{\rm SUM}}

\def\BS1{{\rm BS1}}
\def\BS{{\rm BS}}

\def\ok{\overline{k}}
\def\oj{\overline{j}}
\def\lsa{{\rm LSA}}
\def\newlsa{{\rm ELSA}}
\def\module{\mathcal{M}}
\def\maskmove{{\rm MskMov}}
\def\oM{\overline{{\bf M}}}

\begin{document}

\maketitle

\begin{abstract}
In-context learning is a remarkable property of transformers and has
been the focus of recent research. An attention mechanism is a key
component in transformers, in which an attention matrix encodes
relationships between words in a sentence and is used as weights for
words in a sentence. This mechanism is effective for capturing language
representations. However, it is questionable whether naive
self-attention is suitable for in-context learning in general tasks,
since the computation implemented by self-attention is somewhat
restrictive in terms of matrix multiplication. In fact, we may need
appropriate input form designs when considering heuristic implementations
of computational algorithms.  In this paper, in case of linear
self-attention, we extend it by introducing a bias matrix in addition to
a weight matrix for an input. Despite the simple extension, the extended
linear self-attention can output any constant matrix, input matrix and
multiplications of two or three matrices in the input. Note that
the second property implies that it can be a skip connection. Therefore,
flexible matrix manipulations can be implemented by connecting the
extended linear self-attention components. As an example of
implementation using the extended linear self-attention, we show a
heuristic construction of a batch-type gradient descent of ridge
regression under a reasonable input form.

\end{abstract}

\section{Introduction}

In-context learning is a remarkable property of transformers which are
the basis of large language models such as GPT-3\cite{Brown2020}, and
has been the focus of recent research. In-context learning of
transformers is that, for a prompt containing examples from a task and a
new query input, the trained language model can generate a corresponding
output for the new query input in a zero-shot manner. There are a lot of
research directions of in-context learning as summarized in
\cite{Dong2024}. In this paper, we focus on the analysis of learning
mechanism listed in \cite{Dong2024}.

\cite{Garg2023} has empirically studied the in-context learning
abilities of transformers for various function classes in machine
learning including a linear function class. In particular, in the case
of linear functions, a trained transformer gives a performance
comparable to the least squares solution. \cite{Oswald2022} has gave an
explicit construction of a linear self-attention layer that implements a
single step of a gradient descent algorithm on a mean squared error
loss.  Additionally, it has empirically shown that several
self-attention layers can iteratively perform curvature correction
improving on plain gradient descent algorithm. \cite{Akyurek2022} has
proved that a transformer can implement a gradient descent algorithm and
a closed-form solution for ridge regression. \cite{Dai2023} has also
pointed out the correspondence between the linear version of attention
and a gradient descent algorithm, claiming that transformers perform an
implicit fine-tuning. Also it has empirically investigated a
similarity between in-context learning and explicit
fine-tuning. Although the works of \cite{Garg2023,
Oswald2022,Akyurek2022,Mahankali2023,Dai2023} do not take the training
phase into account, \cite{Zhang2023} has investigated the learning
dynamics of a gradient flow in a simplified transformer architecture
when the training prompts consist of random instances of linear
regression datasets and concluded that transformers trained by a
gradient flow in-context learn a class of linear functions.  More
recently, \cite{Bai2023} has shown that transformers can implement a
broad class of standard machine learning algorithms in-context, such as
least squares, ridge regression and Lasso, using implementations
based on gradient descent algorithms. In contrast to
\cite{Akyurek2022}, \cite{Bai2023} has precisely evaluated the prediction
performance including a network size and shown a near-optimal predictive
power. \cite{Bai2023} has also shown an algorithm selection ability of
transformers; e.g. regularization selection according to validation
error for ridge regression. Besides a regression problem,
\cite{Hataya2024} has shown that transformers can approximate
instance-based and feature-based unsupervised domain adaptation
algorithms and automatically select an algorithm suited for a given
dataset, in which approximation accuracy is also evaluated as in
\cite{Bai2023}. Our work is closely related to
\cite{Garg2023,Oswald2022,Dai2023,Akyurek2022,Bai2023}.

A transformer in \cite{Akyurek2022,Bai2023,Hataya2024} is formed by
stacking a block that consists of a sequential connection of muti-head
self-attention, a one hidden layer network and a skip connection. A one
hidden layer network is used to represent, for example, nonlinear
functions for fitting samples or loss functions required in the
implementation of an algorithm; e.g. \cite{Bai2023,Hataya2024}.  Also,
in \cite{Akyurek2022}, it is used for a multiplication operation
required when implementing ridge regression. Indeed, as in our appendix,
it can be used for a division operation. These rely on the universal
approximation property of layered neural networks; e.g.
\cite{Hornik1989}. A skip connection is also important because it brings
previous data in the stacking block and working spaces to be updated.
Typically, if a block corresponds to a single update step of a gradient
descent algorithm, then the next block needs an original input data and
an updated parameter vector obtained in a current block; e.g.
\cite{Akyurek2022,Bai2023}. These two components in a transformer are
employed also in convolutional neural network\cite{ImageNet2012} and
residual neural networks\cite{He2015} before transformers.

On the other hand, an attention mechanism is unique for transformers
among layered neural networks\cite{Transformer2017}.  Therefore,
historically, the attention mechanism may be a key component for
achieving in-context learning. In self-attention, an attention matrix
encodes relationships among words in a sentence and is used as weights
of words in a sentence. This mechanism may be effective to generate a
required and natural sentence since a certain set of words tends to
appear simultaneously in a sentence. Therefore, the attention mechanism
is valid for capturing language representations. However, it is
questionable whether naive self-attention is suitable for achieving
in-context learning in general tasks. Although most of studies of
in-context learning focus on what transformers can
do\cite{Garg2023,Oswald2022,Akyurek2022,Bai2023,Hataya2024}, this paper
consider what is helpful for in-context learning.

In a gradient descent algorithm for linear regression such as the least
squares and ridge regression, a new coefficient vector is obtained by
adding an update term to a current coefficient vector. We need matrix
multiplications for computing the update term and addition for
generating the new coefficient vector. Addition for the update is
implemented by a skip connection and computation of the update term is
implemented by a linear self-attention
(LSA)\cite{Akyurek2022,Bai2023}. However, roughly speaking, since LSA
consists of multiplication of three matrices which are key, query and
value, it may be restrictive for implementing, for example, 
multiplication of two matrices. Actually, implementation of a gradient
descent algorithm requires an appropriate design of input form to avoid
this restriction; e.g. see \cite{Akyurek2022,Bai2023}. Therefore, for
in-context learning, a component that flexibly manipulates matrix
multiplications may be required. In this paper, by moving focus away
from transformers, we extend linear self-attention to more flexible
component of in-context learning. Although an input matrix (prompt) is
multiplied by a weight matrix in a naive LSA, it is further added a bias
matrix in the extended linear self-attention (ELSA).  ELSA reduces to
LSA if the bias matrices are set to zero matrices. Despite the simple
extension, ELSA can output any constant matrix, input matrix and 
multiplications of two or three matrices in an input. Note that the
second property implies that it can be a skip connection. Therefore,
flexible matrix manipulations can be implemented by connecting ELSAs.
We should note that attention is composed by matrix multiplications and
matrix multiplications are not used in convolutional and residual
neural networks. Therefore, matrix multiplication may be important as a
basic computational component in in-context learning. In this paper, as
an example, we show an ELSA implementation of a batch-type gradient
descent algorithm for ridge regression, in which it is found that ELSA
can adapt a reasonable input form.

The organization of this paper is as follows. In section 2, we formulate
the extended linear self-attention. In section 3, we apply the extended
linear self-attention to implement a batch-type gradient descent
algorithm of ridge regression. Section 4 is devoted for conclusions and
future works.

\section{Extended linear self-attention}

\subsection{Some notations}

We define an $m\times m$ identity matrix by $\I_m$ and an $m\times n$
zero matrix by $\O_{m,n}$.  Let $\A$ be an $m\times n$ matrix. The
$(i,j)$-entry of $\A$ is denoted by $\A[i,j]$.  For $1\le i\le j\le m$
and $1\le k\le l\le n$, we write $\A[i\:j,k\:l]$ for a submatrix of $\A$
comprising the intersection of rows $i$ to $j$ and columns $k$ to $l$.
Also we use $[i\:j,k\:l]$ as a set of positions. In particular, we denote
the $i$-th row vector by $\A[i,\:]$ and $k$-th column vector by
$\A[\:,k]$.

\subsection{Mask and move operation for submatrix}

We here explain a matrix operation that is essentially the same as
``mov'' in \cite{Akyurek2022}. It is pointed out that this operation is
important for in-context learning in \cite{Akyurek2022,Bai2023} and 
also plays a key role in our paper.
Let $\A$ be an $m\times n$ matrix. The purpose here is to
construct a matrix whose certain submatrix is a submatrix of $\A$ and
other elements are zeros. In other words, the matrix operation copies a
submatrix of $\A$ to a certain position of an $m\times n$ zero matrix.

For $1\le i_0,j_0\le m$, we define an $m\times m$ matrix
$\W_{(i_0,j_0)}$, in which
\begin{align}
\W_{(i_0,j_0)}[i,j]=\begin{cases}
1 & (i,j)=(i_0,j_0)\\
0 & {\rm otherwise}
        \end{cases}. 
\end{align}
We define $\P:=\W_{(i_0,j_0)}\A$ whose size is $m\times n$.  If
$i\neq i_0$ then $\W_{(i_0,j_0)}[i,:]=\O_{1,n}$. Therefore,
$\P[i,j]=\W_{(i_0,j_0)}[i,:]\A[:,j]=0$ for any $1\le j\le n$. If $i=i_0$ then
$\P[i_0,j]=\W_{(i_0,j_0)}[i_0,:]\A[:,j]=\A[j_0,j]$ for any $1\le j\le n$.
Therefore, we have $\P[i_0,\:]=\A[j_0,\:]$ and $\P[i,\:]=\O_{1,n}$ for
$i\neq i_0$.  This is extended to the multi-row case.  Let
$K=\{(i_k,j_k):k=1,\ldots,\ok\}$ for $1\le \ok\le m$ be a set of pairs
of matrix indices, for which $1\le i_k,j_k\le m$ and $i_1\neq\cdots\neq
i_{\ok}$. We define an $m\times m$ matrix $\W_K$, in which
\begin{align}
\W_K[i,j]=
\begin{cases}
1 &  (i,j)\in K\\
0 & {\rm otherwise}
\end{cases}.
\end{align}
We define $\P:=\W_K\A$ whose size is $m\times n$. By the above single row
case, it is easy to check that $\P[i_k,\:]=\A[j_k,\:]$ for $1\le k\le \ok$
and $\P[i,\:]=\O_{1,n}$ for $i\notin\{i_1,\ldots,i_{\ok}\}$.

On the other hand, for $1\le k_0,l_0\le n$, 
we define an $n\times n$ matrix $\V_{(k_0,l_0)}$, in which
\begin{align}
\V_{(k_0,l_0)}[k,l]=\begin{cases}
1 & (k,l)=(k_0,l_0)\\
0 & {\rm otherwise}
        \end{cases}. 
\end{align}
We define $\P:=\A\V_{(k_0,l_0)}$ whose size is $m\times n$.
If $l\neq l_0$ then $\V_{(k_0,l_0)}[:,l]=\O_{m,1}$. Therefore, 
$\P[k,l]=\A[k,:]\V_{(k_0,l_0)}[:,l]=0$ for any $1\le k\le m$. If $l=l_0$ then
$\P[k,l_0]=\A[k,:]\W_{(k_0,l_0)}[:,l]=\A[k,k_0]$ for any $1\le k\le m$. Therefore, 
we have $\P[\:,l_0]=\A[\:,k_0]$ and $\P[\:,l]=\O_{m,1}$ for $l\neq l_0$.  
This is extended to the multi-column case.
Let $J=\{(k_j,l_j):j=1,\ldots,\oj\}$ for $1\le \oj\le n$ be a set of pairs of
matrix indices for which $1\le k_j,l_j\le n$ and $l_1\neq\cdots\neq
l_{\oj}$. 
We define an $n\times n$ matrix $\V_J$, in which
\begin{align}
\V_J[k,l]=
\begin{cases}
1 &  (k,l)\in J\\
0 & {\rm otherwise}
\end{cases}.
\end{align}
We define $\P:=\A\V_K$ whose size is $m\times n$. By the above single column
case, it is easy to check that 
$\P[\:,l_j]=\A[\:,k_j]$ for $j=1,\ldots,\oj$ and $\P[\:,l]=\O_{m,1}$ for
$l\notin\{l_1,\ldots,l_{\oj}\}$.

By combining these row and column manipulations, we consider a case of submatrix of
$\A$. Let $(i,j,k,l,a,b)$ be a set of indices that satisfy $1\le i\le
j\le m$, $1\le k\le l\le n$, $1\le i+a,j+a\le m$ and $1\le k+b,~l+b\le
n$. We define a matrix manipulation $\maskmove_{i,j,k,l}^{m,n,a,b}(\A)$
which yields a matrix whose size is the same as $\A$ and, for
$\P:=\maskmove_{i,j,k,l}^{m,n,a,b}(\A)$,
$\P[i+a\:j+a,k+b\:l+b]=\A[i\:j,k\:l]$ holds and the elements out of
$[i+a\:j+a,k+b\:l+b]$ are zeros. The following summary is mainly used
in this paper.
\begin{lemma}
\label{lemma:mskmov} 
Let $\A$ be an $m\times n$ matrix. There exist an
$m\times m$ matrix $\W$ and an $n\times n$ matrix $\V$ such that
\begin{align}
\label{eq:mskmov-by-mat}
\maskmove_{i,j,k,l}^{m,n,a,b}(\A)=\W\A\V.
\end{align}
\end{lemma}
\begin{proof}
By choosing $K:=\{(i'+a,i'):i\le i'\le j\}$ and $J:=\{(k',k'+b):k\le
k'\le l\}$ in the previous multi-row and multi-column cases,
(\ref{eq:mskmov-by-mat}) is obvious if we set $\W=\W_K$ and $\V=\V_J$.
\end{proof}

$\maskmove_{i,j,k,l}^{m,n,a,b}(\A)$ is viewed as an operation, in which
a submatrix of $\A$ is inserted into a certain position of $\O_{m,n}$.
On the other hand, let $\odot$ denote the Hadamard product.  We define
$\M_{i,j,k,l}^{m,n}$ as an $m\times n$ matrix whose $(s,t)$-entry is $1$
if $1\le i\le s\le j\le m,~1\le k\le t\le l\le n$ and $0$
otherwise. This is viewed as a mask for a submatrix. Then,
$\P:=\M_{i,j,k,l}^{m,n}\odot\A$ implements a masking operation, in which
$\P[i\:j,k\:l]=\A[i\:j,k\:l]$ and $\P[s,t]=0$ if $(s,t)$ is not in
$[i\:j,k\:l]$. Therefore, $\maskmove_{i,j,k,l}^{m,n,a,b}$ implements an
operation that masks a submatrix and move the masked submatrix to a
certain position. We refer to $\maskmove_{i,j,k,l}^{m,n,a,b}$ as ``mask
and move'' operation. The mask and move operation plays a key role in the
linear self-attention defined in this paper.

We show a typical example, in which a submatrix is inserted into the
upper right portion of a zero matrix. Let a target submatrix be
$\A[i\:j,k\:l]$, in which $1\le i\le j\le m$ and $1\le k\le l\le n$.  We
set $\ok=j-i+1$ and $\oj=l-k+1$.  We choose
$K=\{(1,i),(2,i+1),\ldots,(\ok,j)\}$; i.e. $a=-(i-1)$ in the proof of
Property \ref{lemma:mskmov}.  We also choose
$J=\{(k,n-\oj+1),(k+1,n-\oj+2),\ldots,(l,n)\}$; i.e. $b=(n-l)$ in the
proof of Property \ref{lemma:mskmov}.  Let $\W$ and $\V$ be an $m\times m$
matrix and $n\times n$ matrix respectively. If we set
\begin{align}
\W[i,j]=
\begin{cases}
1 & (i,j)\in K\\
0 & {\rm otherwise} 
\end{cases},~ 
\V[i,j]=
\begin{cases}
1 & (i,j)\in J\\
0 & {\rm otherwise} 
\end{cases}.
\end{align}
then it is easy to check that
\begin{align}
\W\A\V=\maskmove_{i,j,k,l}^{m,n,-(i-1),n-l}(\A)
\end{align}
holds.

\subsection{Linear self-attention}

Let $\H$ be an $m\times n$ input matrix. We define a
component whose output for $\H$ is given by
\begin{align}
\label{eq:def-lsa}
\lsa_{\vtheta}(\H)
=(\H\W_3)(\H\W_1)^\top(\H\W_2)=(\H\W_3)(\W_1^\top\H^\top\H\W_2),
\end{align}
where $\vtheta=\{\W_1,\W_2,\W_3\}$ is an ordered set of parameters, in
which $\W_1$, $\W_2$ and $\W_3$ are $n\times n$ weight matrices.  Note
that
\begin{align}
\lsa_{\vtheta}(\H)^\top=
(\W_2^\top\H^\top)(\W_1^\top\H^\top)^\top(\W_3^\top\H^\top),
\end{align}
is an output of original linear self-attention defined in
\cite{Bai2023}, in which input sequence is $\H^\top$ and $\W_1^\top$,
$\W_3^\top$ and $\W_2^\top$ are weight matrices for key, query and
value. In this paper, we refer to the above $\lsa_{\vtheta}$ as a linear
self-attention (LSA).  Note that $(\W_1^\top\H^\top\H\W_2)$ in LSA is
$\maskmove$ operation for $\H^\top\H$.  For example, we consider a case
in which $\H$ has a form of $\H=[\X_1,\ldots,\X_k]$, where $\X_j$ is
regarded as a vector or matrix valued variable for $j=1,\ldots,k$.  In
this case, $\maskmove$ helps for extracting submatrices of $\H^\top\H$,
which may represent correlation structures among variables in $\H$;
i.e. $\X_i^\top\X_j$ for $1\le i,j\le k$.

We here consider to compute matrix multiplication by LSA defined
above. Let $\A$ and $\B$ are $r\times s$ and $s\times t$ matrices. Our
purpose here is to compute $\A\B$ and store it in an output matrix by
using LSA when $\A$ and $\B$ are submatrices of an input matrix
$\H$. Roughly speaking, since we multiply three times $\H$ in LSA,
obviously, we need to design $\H$ to extract $\A\B$ by LSA. For example,
we set
\begin{align}
\H=\begin{bmatrix}
    \A & \O_{r,t} & \O_{r,s}\\
\O_{s,s} & \B & \I_s
   \end{bmatrix},
\end{align}
which is an $(r+s)\times(2s+t)$ matrix. We have
\begin{align}
\H^\top\H=\begin{bmatrix}
\A^\top\A & \O_{s,t} & \O_{s,s}\\
\O_{t,s} & \B^\top\B & \B^\top\\
\O_{s,s} & \B & \I_s
   \end{bmatrix}. 
\end{align}
Therefore, by appropriately choosing 
$(2s+t)\times(2s+t)$ matrices $\W_1$ and $\W_2$ according to
Property \ref{lemma:mskmov}, we have
\begin{align}
\W_1^\top\H^\top\H\W_2=\begin{bmatrix}
\O_{s,2s} & \B\\
\O_{t+s,2s} & \O_{t+s,t}
   \end{bmatrix}. 
\end{align}
On the other hand, by setting
\begin{align}
\W_3=\begin{bmatrix}
      \I_s & \O_{s,s+t}\\
\O_{s+t,s} & \O_{s+t,s+t}
     \end{bmatrix}, 
\end{align}
we have 
\begin{align}
\H\W_3=\begin{bmatrix}
        \A & \O_{r,s+t}\\
\O_{s,s} & \O_{s,s+t}
       \end{bmatrix}.
\end{align}
As a result, we obtain
\begin{align}
\lsa_\vtheta(\H)=
\begin{bmatrix}
\O_{r,2s} & \A\B\\
\O_{s,2s} & \O_{s,t}
   \end{bmatrix}. 
\end{align}
In this construction, we insert an extra identity matrix in the input
matrix. It is used for direct extraction of $\B$ that is a submatrix of
$\H$. If the extra identity matrix is not
included in the input matrix then $\A\B$ is possibly multiplied by
obstacles. This inconvenience of LSA is relaxed in the following extended
linear self-attention.

\subsection{Extended linear self-attention}

We introduce bias matrices in addition to weight matrices in 
a transformation of an input matrix in LSA.

We define a component whose output for an $m\times n$ input matrix $\H$ is
given by
\begin{align}
\label{eq:def-extended-lsa}
\newlsa_{\vtheta}(\H)&=(\H\W_3+\B_3)(\H\W_1+\B_1)^\top(\H\W_2+\B_2)
\end{align}
where $\vtheta=\{(\W_l,\B_l):l=1,2,3\}$ is an ordered set of parameters,
in which $\W_l$ and $\B_l$ are an $n\times n$ weight matrix and an
$m\times n$ bias matrix respectively.  This is obviously an extension of
LSA since it reduces to LSA if we set $\B_1=\B_2=\B_3=\O_{m,n}$.  We
refer to (\ref{eq:def-extended-lsa}) as an extended linear
self-attention (ELSA). We can show that ELSA can output any constant
matrix, input matrix and a multiplication of two matrices.

We first consider an implementation for outputting an arbitrary constant matrix.
\begin{lemma}
\label{lemma:elsa-const}
Let $\C$ be an $m\times n$ constant matrix.  For any $m\times n$
input matrix $\H$, there exists $\vtheta$, by which
$\newlsa_{\vtheta}(\H)=\C$ holds.
\end{lemma}
\begin{proof}
 If we set $\W_k=\O_{m,m}$ for $k=1,2,3$ then $\newlsa(\H)=\B_3\B_1^\top\B_2$.
In case of $m>n$, by setting
\begin{align}
\B_1=\B_2=\begin{bmatrix}
      \I_n\\
\O_{m-n,n}
     \end{bmatrix},
\end{align}
we obtain $\B_1^\top\B_2=\I_n$. By setting $\B_3=\C$, we have
$\newlsa_\vtheta(\H)=\C\I_n=\C$. In case of $m<n$, by setting
\begin{align}
\B_1=\B_3=\begin{bmatrix}
      \I_m & \O_{m,n-m}
     \end{bmatrix},
\end{align}
we obtain $\B_3\B_1^\top=\I_m$. By setting $\B_2=\C$, we have
$\newlsa_\vtheta(\H)=\I_m\C=\C$.  In case of $m=n$,
$\newlsa_\vtheta(\H)=\C$ holds, for example, by choosing $\B_1=\B_2=\I_m$ and
$\B_3=\C$.
\end{proof}
Therefore, for any input matrix, ELSA can output any constant matrix
whose size is the same as that of the input matrix. 

Next ,we consider an implementation for outputting the input matrix.  It
implies that ELSA can work as a skip connection.
\begin{lemma}
\label{lemma:elsa-skip}
For any $m\times n$ input matrix $\H$, there exists $\vtheta$, by which
$\newlsa_{\vtheta}(\H)=\H$ holds. 
\end{lemma}
\begin{proof}
In case of $m>n$, by setting 
\begin{align}
\B_1=\B_2=\begin{bmatrix}
      \I_n\\
\O_{m-n,n}
     \end{bmatrix},
\end{align}
we have $\B_1^\top\B_2=\I_n$. By setting 
$\W_1=\W_2=\O_{n,n}$, $\W_3=\I_n$ and
$\B_3=\O_{m,n}$, we have $\newlsa_\vtheta(\H)=\H\W_3\B_1^\top\B_2=\H\I_n\I_n=\H$. 
In case of $m<n$, by setting
\begin{align}
\B_1=\B_3=\begin{bmatrix}
      \I_m & \O_{m,n-m}
     \end{bmatrix},
\end{align}
we obtain $\B_3\B_1^\top=\I_m$. By setting $\W_1=\W_3=\O_{n,n}$,
$\W_2=\I_n$ and $\B_2=\O_{m,n}$, we have
$\newlsa_\vtheta(\H)=\B_3\B_1^\top\H\W_2=\I_m\H\I_n=\H$.  In case of
$m=n$, for example, by setting $\W_1=\W_2=\O_{m,m}$, $\W_3=\I_m$,
$\B_1=\B_2=\I_m$ and $\B_3=\O_{m,n}$, we have
$\newlsa_\vtheta(\H)=\H\W_3\B_1^\top\B_2=\H\I_m\I_m\I_m=\H$.
\end{proof}
Therefore, for any input matrix, ELSA can play a role of a skip connection.

We finally consider an implementation of a multiplication of two matrices.
\begin{lemma}
\label{lemma:elsa-mul} Let $\A$ and $\B$ are $r\times s$ and $s\times t$
matrices. Under an appropriate choice
of an $m\times n$ input matrix $\H$ that includes $\A^\top$ and $\B$ (or
$\A$ and $\B$) as submatrices, there exists $\vtheta$, by which
$\newlsa_{\vtheta}(\H)$ includes $\A\B$.
\end{lemma}
\begin{proof}
Since $\A\B$ is an $r\times t$ matrix, a size of output should be larger
than $r\times t$.  Since the size of output is the same as that of $\H$,
the size of input should be larger than $r\times t$. 
As a simple case, we choose an $m\times n$ input matrix that is given by
\begin{align}
\H:=\begin{bmatrix}
    \A^\top & \B\\
\O_{r,r} & \O_{r,t}
   \end{bmatrix},
\end{align}
where $m:=s+r$ and $n=:r+t$. We have
\begin{align}
\H^\top\H=
\begin{bmatrix}
\A\A^\top & \A\B\\
\B^\top\A^\top & \B^\top\B
\end{bmatrix}.
\end{align}
By choosing $\W_1$ and $\W_2$ according to
Property \ref{lemma:mskmov}, we have
\begin{align}
\W_1^\top\H^\top\H\W_2=
\begin{bmatrix}
\O_{r,r} & \A\B\\
\O_{t,r} & \O_{t,t}
   \end{bmatrix}. 
\end{align}
We then set $\W_3=\O_{m,m}$ and 
\begin{align}
\B_3=\begin{bmatrix}
      \I_r & \O_{r,t } \\
\O_{s,r} &\O_{s,t}
     \end{bmatrix}. 
\end{align}
Then, output matrix is obtained by
\begin{align}
\newlsa_\vtheta(\H)=
\begin{bmatrix}
\O_{r,r} & \A\B\\
\O_{s,r} & \O_{s,t}
   \end{bmatrix}. 
\end{align}

On the other hand, we consider a case that an $(r+s)\times(s+t)$ input
matrix is given by
\begin{align}
\H:=\begin{bmatrix}
    \A & \O_{r,t}\\
\O_{s,s} & \B
   \end{bmatrix}.
\end{align}
By setting $\B_3=\B_2=\O_{m,n}$ and
\begin{align}
\W_3=\begin{bmatrix}
\O_{s,t} & \I_s\\
\O_{t,t} & \O_{t,s}
\end{bmatrix},~
\W_2=\begin{bmatrix}
\O_{s,s} & \O_{s,t}\\
\O_{t,s} & \I_t
\end{bmatrix},
\end{align}
we have 
\begin{align}
\H\W_3=\begin{bmatrix}
\O_{r,t} & \A\\
\O_{s,t} & \O_{s,s}
\end{bmatrix},~
\H\W_2=\begin{bmatrix}
\O_{r,s} & \O_{r,t}\\
\O_{s,s} & \B
\end{bmatrix}.
\end{align}
By setting $\W_1=\O_{s+t,s+t}$ and 
\begin{align}
\B_1=\begin{bmatrix}
\O_{r,t} & \O_{r,s}\\
\O_{s,t} & \I_s
\end{bmatrix},
\end{align}
we obtain
\begin{align}
\newlsa_\vtheta(\H)=(\H\W_3)\B_1^\top(\H\W_2)=
\begin{bmatrix}
\O_{r,s} & \A\B\\
\O_{s,s} & \O_{s,t}
   \end{bmatrix}. 
\end{align}
\end{proof}

\section{Implementation of ridge regression}

\subsection{Ridge regression}

Let $\{(\x_i,y_i):i=1,\ldots,n\}$ be $n$ pairs of input-output examples,
where $\x_i=(x_{i,1},\ldots,x_{i,d})^\top\in\R^d$ and
$y_i\in\R$.  If necessary, we set $x_{i,1}=1$ for a constant term and
the number explanatory variables is $d-1$ in this case.
We also have a new input denoted by
$\u=(u_1,\ldots,u_d)^\top$ and expect to predict the
corresponding output by applying linear regression. Let $\X$ be an
$n\times d$ matrix whose $(i,j)$ entry is $x_{i,j}$. Therefore, the row
vector of $\X$ is $\x_i^\top$. In linear regression, a model output for an input
$\x=(x_1,\ldots,x_d)^\top$ is given by
\begin{align}
\label{eq:regressor}
f_{\w}(\x):=\w^\top\x=\sum_{k=1}^dw_kx_k,
\end{align}
where $\w=(w_1,\ldots,w_d)^\top\in\R^d$ is a coefficient vector.  We
define $\y=(y_1,\ldots,y_n)^\top$.  The ridge estimator with a ridge
parameter $\lambda\ge 0$ is given by
\begin{align}
\label{eq:ridge-estimator}
\evw_{\lambda}=(\X^\top\X+\lambda\I_d)^{-1}\X^\top\y,
\end{align}
which can be viewed as a solution of a system of linear equations with
the form of $\F\w_{\lambda}=\b$, where $\F=\X^\top\X+\lambda\I_d$ and
$\b=\X^\top\y$. Note that the ridge estimator is the minimizer of the
$\ell_2$ regularized cost function defined by
\begin{align}
\label{eq:l2-reg-cost}
\ell(\w)=\frac{1}{2}\|\y-\X\w\|^2+\frac{\lambda}{2}\|\w\|^2. 
\end{align}
If $\lambda=0$ and $\X^\top\X$ is non-singular, $\evw_0$ is the least
squares estimator.  In considering the least squares
estimation, we always assume that $d\le n$ and the rank of $\X$ is
$d$.

After the estimation, a predicted output for $\u$ is given by
$\u^\top\evw_{\lambda}$.  Our purpose is to construct a layered network that outputs
$\u^\top\evw_{\lambda}$ for any given $(\X,\y,\lambda,\u)$. 

\subsection{Implementation of a closed-form solution}

A solver for a system of linear equations is required to obtain a
closed-form solution for ridge regression. In the appendix, we
have implemented Gaussian elimination for solving systems of linear
equations using standard components which are a ReLU
network with one hidden layer, matrix multiplication and skip
connections. Unfortunately, this implementation consists of a combination
of these components while it is not represented by a layered structure.

\cite{Akyurek2022} has implemented a closed-form of ridge estimation
using a transformer. The important point in \cite{Akyurek2022} is that
several computational primitives are explicitly given and the solver of
ridge regression is expressed as a processing or operation using the
primitives.  In other words, it implements an algorithm that outputs a
ridge estimate by using the computational primitives. In the
implementation of the solver of ridge regression, we need multiplication
and division operations. In \cite{Akyurek2022}, multiplication is
realized by using the GeLU (Gaussian Error Linear Unit) nonlinearity
based on Taylor's expansion; see also \cite{Liang2017}. In the appendix,
we implement it by matrix multiplication. On the other hand, a
division operation is essential since we need the matrix inverse to
obtain a ridge estimate.  In \cite{Akyurek2022}, the Sherman–Morrison
formula is used to reduce the matrix inversion to a sequence of rank-one
updates performed example-by-example, and then the division operation is
implemented by a ``LayerNorm'' function. In the appendix, we implement
it approximately by a ReLU network with one hidden layer.

When considering the implementation of an algorithm on a layered
structure, a sequential connection of simple modules with the same
structure is preferable. In this sense, iteration-based algorithms are
preferable to specific algorithms such as Gaussian elimination; e.g.
see \cite{Bai2023}. Moreover, in the case of ridge regression, the
gradient descent update does not require nonlinearity, but only requires
matrix multiplication and addition; e.g.
\cite{Mahankali2023,Oswald2022}.  Here we consider LSA and ELSA
implementations of a batch-type gradient descent algorithm for ridge regression.

\subsection{Gradient descent}

We consider to obtain a solution to ridge regression by a batch-type
gradient descent algorithm.  Since we have
\begin{align}
\frac{\partial \ell(\w)}{\partial\w}=-\X^\top(\y-\X\w)+\lambda\w
\end{align}
by (\ref{eq:l2-reg-cost}), the batch update at an iteration $t$ is given by
\begin{align}
\label{eq:ridge-gd}
\w_t&=\w_{t-1}-\eta\Delta\w_{t-1}\\
\label{eq:ridge-gd-delta}
\Delta\w_{t-1}&=-\X^\top\y+\X^\top\X\w_{t-1}+\lambda\w_{t-1}, 
\end{align}
where $\eta>0$ is a learning rate. Let $T$ be the number of iterations.
Therefore, $t=1,\ldots,T$ and $\w_0$ is an initial coefficient vector.

In the gradient descent algorithm, (\ref{eq:ridge-gd}) only requires
matrix multiplication and addition, by which a coefficient vector
obtained after enough iterations is an approximation of ridge
solution. Note that this is viewed as approximating the matrix inversion
calculation by an infinite iteration of a linear calculation. If we can
construct a module that performs (\ref{eq:ridge-gd}), then the gradient
descent algorithm is implemented by a layered network structure with
stacking of the module.

Let $\H_0$ be an input matrix which has a certain form including
information on $\X$, $\y$, $\u$, $\w_0$, $\lambda$ and $\eta$.  We then
consider to construct a layered network which receives $\H_0$ and output
$\u^\top\w_T$, in which $\w_T$ is a coefficient vector after $T$ iterations
of (\ref{eq:ridge-gd}).  To do this, we construct a module $\module_t$
that receives $\H_{t-1}$ and output $\H_t$; i.e.
$\H_t=\module_t(\H_{t-1})$. And, at a final output, $\H_{T+1}$ includes
$\u^\top\w_T$.

\subsection{Implementation under a designed input form}

We show an example of the implementation of a batch-type gradient
descent algorithm for ridge regression by using a naive LSA; e.g. see
also \cite{Akyurek2022,Bai2023}.

An input matrix form at the $t$-th step is assumed to be
\begin{align}
\label{eq:designed-input}
\H_{t-1}:=\begin{bmatrix}
    \sqrt{\eta}\X^\top & \O_{d,n} & \O_{d,1} & \sqrt{\eta}\sqrt{\lambda}\I_d & 
\u & \w_{t-1}\\
\O_{1,n} & \sqrt{\eta}\y^\top & 1 & \O_{1,d} & 0 & 0
   \end{bmatrix},
\end{align}
which is an $(d+1)\times s$ matrix, where we define $s:=2n+d+3$.  In
$\H_t$, $\w_{t-1}$ is an updated coefficient vector at the $(t-1)$-th
step and members other than $\w_{t-1}$ are fixed for any $t$.

We employ a module $\module_t$ that consists of a multi-head LSA and a
skip connection. Note that we do not employ any nonlinearity. For a
$(d+1)\times s$ input matrix $\H_{t-1}$, we define
\begin{align}
\label{eq:mod-input-design}
\module_t(\H_{t-1})&:=\P_t+\H_{t-1}\\
\label{eq:mhlsa-input-design}
\P_t=\P_t(\H_{t-1})&:=\sum_{k=1}^{\ok}\lsa_{\vtheta_{t,k}}(\H_{t-1})
\end{align}
where $\ok$ is the number of heads and $\lsa_{\vtheta_{t,k}}$ is defined
in (\ref{eq:def-lsa}), in which
$\vtheta_{t,k}=\{\W_{t,k,1},\W_{t,k,2},\W_{t,k,3}\}$. $\W_{t,k,l}$,
$l=1,2,3$ are $s\times s$ weight matrices. $\P_t$ is an output of
multi-head LSA and is an $(d+1)\times s$ matrix. In our implementation,
the weight matrices do not depend on $t$ for $t=1,\ldots,T$; i.e. those
are common for every gradient descent steps. Therefore, we write
$\{\W_{k,1},\W_{k,2},\W_{k,3}\}=\{\W_{t,k,1},\W_{t,k,2},\W_{t,k,3}\}$
for $t=1,\ldots,T$. We set $\ok=3$ below.

Now, we have
\begin{align}
\label{eq:Htop-H-input-design}
&\H_{t-1}^\top\H_{t-1}\notag\\
&=
\begin{bmatrix}
\eta\X\X^\top & \O_{n,n} & \O_{n,1} & \eta\sqrt{\lambda}\X^\top & 
\sqrt{\eta}\X\u & \sqrt{\eta}\X\w_{t-1}\\
\O_{n,n} & \eta\y\y^\top & \sqrt{\eta}\y & \O_{n,d} & \O_{n,1} & \O_{n,1}\\
\O_{1,n} & \sqrt{\eta}\y^\top & 1 & \O_{1,d} & 0 & 0\\
\eta\sqrt{\lambda}\X^\top & \O_{d,n} & \O_{d,1} & \eta\lambda\I_d & 
\sqrt{\eta}\sqrt{\lambda}\u & \sqrt{\eta}\sqrt{\lambda}\w_{t-1}\\
\sqrt{\eta}\u^\top\X^\top & \O_{d,n} & 0 & \sqrt{\eta}\sqrt{\lambda}\u^\top 
& \u^\top\u & \u^\top\w_{t-1}\\
\sqrt{\eta}\w^\top\X^\top & \O_{d,n} & 0 & \sqrt{\eta}\sqrt{\lambda}\w_{t-1}^\top & 
\w_{t-1}^\top\u & \w_{t-1}^\top\w_{t-1}\\
\end{bmatrix}, 
\end{align}
which is a $s\times s$ matrix. Since we have
\begin{align}
(\H_{t-1}\W_{k,1})^\top\H_{t-1}\W_{k,2}
=\W_{k,1}^\top(\H_{t-1}^\top\H_{t-1})\W_{k,2}
\end{align}
for $1\le k\le \ok$, 
by choosing $\W_{k,1}$ and $\W_{k,2}$ 
for (\ref{eq:Htop-H-input-design}) at
each $k$ according to Property \ref{lemma:mskmov}, we have
\begin{align}
\label{eq:lsa-maskmove-1-1}
(\H_{t-1}\W_{1,1})^\top\H_{t-1}\W_{1,2}
&=\begin{bmatrix}
   \O_{n,s-1} & \sqrt{\eta}\y\\
   \O_{s-n,s-1} & \O_{s-n,1}
  \end{bmatrix}\\
(\H_{t-1}\W_{2,1})^\top\H_{t-1}\W_{2,2}
&=\begin{bmatrix}
   \O_{n,s-1} & \sqrt{\eta}\X\w_{t-1}\\
   \O_{s-n,s-1} & \O_{s-n,1}
  \end{bmatrix}\\
(\H_{t-1}\W_{3,1})^\top\H_{t-1}\W_{3,2}
&=\begin{bmatrix}
   \O_{d,s-1} & \sqrt{\eta}\sqrt{\lambda}\w_{t-1}\\
   \O_{s-d,s-1} & \O_{s-d,1}
  \end{bmatrix}.
\end{align}
For example, in (\ref{eq:lsa-maskmove-1-1}), we can extract $\sqrt{\eta}\y$
 by using the mask and move operation in LSA since it is included in 
 $\H_{t-1}^\top\H_{t-1}$. The same computations appear in below.
By setting 
\begin{align}
\W_{1,3}&=\begin{bmatrix}
    \I_n & \O_{n,s-n}\\
\O_{s-n,n} & \O_{s-n,s-n}
   \end{bmatrix},~
\W_{3,3}=
\begin{bmatrix}
\O_{s-d-2,d} & \O_{s-d-2,s-d}\\
   -\I_d & \O_{d,s-d}\\
\O_{2,d} & \O_{2,s-d}
   \end{bmatrix}
\end{align}
and $\W_{2,3}=-\W_{1,3}$, we have
\begin{align}
\H_{t-1}\W_{1,3}&=
\begin{bmatrix}
    \sqrt{\eta}\X^\top & \O_{d,s-n}\\
\O_{1,n} & \O_{1,s-n}
   \end{bmatrix}\\
\H_{t-1}\W_{2,3}&=\begin{bmatrix}
    -\sqrt{\eta}\X^\top & \O_{d,s-n}\\
\O_{1,n} & \O_{1,s-n}
   \end{bmatrix}\\
\H_{t-1}\W_{3,3}&=\begin{bmatrix}
   - \sqrt{\eta}\sqrt{\lambda}\I_d & \O_{d,s-d}\\
\O_{1,d} & \O_{1,s-d}
   \end{bmatrix}.
\end{align}
Therefore, we obtain
\begin{align}
\P_t=
 \sum_{k=1}^3(\H_{t-1}\W_{k,3})
(\H_{t-1}^\top\W_{k,1}^\top)\H_{t-1}\W_{k,2}
=
\begin{bmatrix}
   \O_{d,s-1} & -\eta\Delta\w_{t-1}\\
   \O_{1,s-1} & \O_{s-n,1}
  \end{bmatrix},
\end{align}
where  $\Delta\w_{t-1}$ is defined in (\ref{eq:ridge-gd}).
By (\ref{eq:mod-input-design}), $\H_t:=\module_t(\H_{t-1})$ is obtained by
\begin{align}
\H_t=\P_t+\H_{t-1}=
\begin{bmatrix}
    \sqrt{\eta}\X^\top & \O_{d,n} & \O_{d,1} & \sqrt{\eta}\sqrt{\lambda}\I_d & 
\u & \w_t\\
\O_{1,n} & \sqrt{\eta}\y^\top & 1 & \O_{1,d} & 0 & 0
   \end{bmatrix},
\end{align}
where $\w_t$ is the updated coefficient vector in (\ref{eq:ridge-gd-delta}).
By applying $\module_t$ sequentially for $t=1,\ldots,T$ under a certain choice
of $\w_0$, we obtain $\H_T=\module_T(\H_{T-1})$ which includes an
estimate of a coefficient vector $\w_T$.  

We finally insert $\u^\top\w_T$ into the output of $\module_{T+1}$.  By
choosing $\W_{T+1,1,1}$ and $\W_{T+1,1,2}$ 
for (\ref{eq:Htop-H-input-design}) according to Property
\ref{lemma:mskmov}, we have
\begin{align}
(\H_T\W_{T+1,1,1})^\top\H_T\W_{T+1,1,2}
&=\begin{bmatrix}
   \O_{s-1,s-1} & \O_{s-1,1}\\
   \O_{1,s-1} & \u^\top\w_T
  \end{bmatrix}.
\end{align}
If we set
\begin{align}
\W_{T+1,1,3}=\begin{bmatrix}
      \O_{2n,s-1} & \O_{2n,1}\\
\O_{1,s-1} & 1\\
\O_{d+2,s-1} & \O_{d+2,1}
     \end{bmatrix} 
\end{align}
then
\begin{align}
\H_T\W_{T+1,1,3}
&=\begin{bmatrix}
   \O_{d,s-1} & \O_{d,1}\\
   \O_{1,s-1} & 1
  \end{bmatrix}.
\end{align}
Therefore, by choosing $\W_{T+1,k,3}=\O_{s,s}$ for $k=2,3$, we obtain
\begin{align}
\P_{T+1}
&=(\H_T\W_{T+1,1,3})(\H_T\W_{T+1,1,1})^\top\H_T\W_{T+1,1,2}\notag\\
&=\begin{bmatrix}
   \O_{d,s-1} & \O_{d,1}\\
   \O_{1,s-1} & \u^\top\w_T
  \end{bmatrix}.
\end{align} 
As a result, by (\ref{eq:mod-input-design}),
$\H_{T+1}:=\module_{T+1}(\H_T)$ is obtained by
\begin{align}
\H_{T+1}&=\P_{T+1}+\H_T\notag\\
&=
\begin{bmatrix}
    \sqrt{\eta}\X^\top & \O_{d,n} & \O_{d,1} & \sqrt{\eta}\sqrt{\lambda}\I_d & 
\u & \w_T\\
\O_{1,n} & \sqrt{\eta}\y^\top & 1 & \O_{1,d} & 0 & \u^\top\w_T
   \end{bmatrix},
\end{align}
in which the right bottom element is a solution that is a predicted output
for $\u$ after $T$ steps.

\subsection{Implementation by extended linear self-attention}

Since ELSA without bias matrix terms is LSA, the previous implementation
under a specific input form is also valid for ELSA.

We here consider the other input form, in which variables are 
enumerated and more generally compared to the previous implementation.
We define $d\times s$ input matrix by
\begin{align}
\label{eq:enumerated-input}
\H_{t-1}=\begin{bmatrix}
\X^\top & \Y_0^\top & \lambda\I_d & \sqrt{\eta}\I_d & \u & \O_{d,1} & \w_{t-1}
   \end{bmatrix},
\end{align}
where $\Y_0:=[\O_{n,d-1}~~\y]$ and $s:=2n+2d+3$.  In $\H_{t-1}$,
$\O_{d,1}$ is used for storing a prediction result. Indeed, the mask and
move operation can extract relationships between variables in this form
and it is a part of ELSA.  In the previous implementation using LSA, in
addition to a specific form, we need to set $\sqrt{\eta}$ as a scaling
factor for some matrices. This may come from a restriction of LSA with
respect to matrix multiplication.  This arrangement is relaxed in
(\ref{eq:enumerated-input}), in which what we need are independently
included.

We employ a module $\module_t$ that consists of a sequential connection
of two multi-head ELSAs and a skip connection. We do not employ any
nonlinearity again.  For a $d\times s$ input matrix $\H_{t-1}$, we
define
\begin{align}
\label{eq:mod-elsa}
\module_t(\H_{t-1})&:=\P_{(2),t}+\H_{t-1}\\
\P_{(2),t}&:=
\sum_{k=1}^{\ok}\newlsa_{\vtheta_{(2),t,k}}(\P_{(1),t}),\\
\P_{(1),t}&:=
\sum_{k=1}^{\ok}\newlsa_{\vtheta_{(1),t,k}}(\H_{t-1}),
\end{align}
where $\ok$ is the number of heads and $\newlsa_{\vtheta_{(j),t,k}}$ for
$j=1,2$ are defined in (\ref{eq:def-extended-lsa}), in which
$\vtheta_{(j),t,k}=\{(\W_{(j),t,k,l},\B_{(j),t,k,l}):l=1,2,3\}$. $\W_{(j),t,k,l}$
is a $s\times s$ weight matrix and $\B_{(j),t,k,l}$ is a $d\times s$ bias matrix.
In our implementation, again, the weight matrices do not depend on $t$
for $t=1,\ldots,T$; i.e. those are common for every steps. Therefore,
we write $\vtheta_{(j),t,k}=\vtheta_{(j),k}$ and
$(\W_{(j),k,l},\B_{(j),k,l})=(\W_{(j),t,k,l},\B_{(j),t,k,l})$ for any
$t=1,\ldots,T$. We set $\ok=4$ below.

We have
\begin{align}
\label{eq:Htop-H-elsa}
&\H_{t-1}^\top\H_{t-1}\notag\\
&=
\begin{bmatrix}
\X\X^\top & \sqrt{\eta}\X\Y_0^\top & \lambda\X &  \sqrt{\eta}\X & 
\X\u & \O_{n,1} & \X\w_{t-1}\\
\Y_0\X^\top & \Y_0\Y_0^\top & \lambda\Y_0 & \sqrt{\eta}\Y_0 & 
\Y_0\u & \O_{n,1} & \Y_0\w_{t-1}\\
\lambda\X^\top & \lambda\Y_0^\top & \lambda^2\I_d & 
\lambda\sqrt{\eta}\I_d & \lambda\u & \O_{d,1} &\lambda\w_{t-1}\\
\sqrt{\eta}\X^\top & \sqrt{\eta}\Y_0^\top & \sqrt{\eta}\lambda\I_d & \eta\I_d & 
\sqrt{\eta}\u & \O_{d,1} & \sqrt{\eta}\w_{t-1}\\
\u^\top\X^\top & \u^\top\Y_0^\top & \lambda\u^\top &
\sqrt{\eta}\u^\top & \u^\top\u & 0 & \u^\top\w_{t-1}\\
\O_{1,n} & \O_{1,n} & \O_{1,d} &
\O_{1,d} & 0 & 0 & 0\\
\w_{t-1}^\top\X^\top & \w_{t-1}^\top\Y_0^\top & \lambda\w_{t-1}^\top &
\sqrt{\eta}\w_{t-1}^\top & \w_{t-1}^\top\u & 0 & \w_{t-1}^\top\w_{t-1}
\end{bmatrix}.
\end{align}
We construct a module that computes $\w_t$ from $\w_{t-1}$. It
consists of two sequential ELSA blocks whose
parameters are
$\{\vtheta_{(1),k}:k=1,2,3,4\}$ for
the first block and $\{\vtheta_{(2),k}:k=1,2,3,4\}$ for the second block.

We show a design for the first block.
\begin{itemize}
\item For $k=1$, we set $\B_{(1),1,l}=\O_{d,s}$ for $l=1,2,3$. By
      choosing $\W_{(1),1,1}$ and $\W_{(1),1,2}$
for (\ref{eq:Htop-H-elsa}) according to Property \ref{lemma:mskmov}, we
      have
\begin{align}
(\H_{t-1}\W_{(1),1,1})^\top\H_{t-1}\W_{(1),1,2}
=\begin{bmatrix}
\O_{n,s-1} & \X\w_{t-1}\\
\O_{s-n,s-1} & \O_{s-n,1}
               \end{bmatrix}. 
\end{align}
By setting
\begin{align}
\W_{(1),1,3}=\begin{bmatrix}
\I_n & \O_{n,s-n}\\
\O_{s-n,n} & \O_{s-n,s-n}
         \end{bmatrix},
\end{align}
we have
\begin{align}
\H_{t-1}\W_{(1),1,3}=\begin{bmatrix}
         \X^\top & \O_{d,s-n}
         \end{bmatrix}.
\end{align}
Therefore, we obtain
\begin{align}
\newlsa_{\vtheta_{(1),1}}(\H_{t-1})
=\begin{bmatrix}
\O_{d,s-1} & \X^\top\X\w_{t-1}
               \end{bmatrix}. 
\end{align}

\item For $k=2$, by setting $\B_{(1),2,1}=\B_{(1),2,2}=\O_{d,s}$ and 
choosing $\W_{(1),2,1}$ and $\W_{(1),2,2}$ for (\ref{eq:Htop-H-elsa})
according to Property \ref{lemma:mskmov}, we have
\begin{align}
(\H_{t-1}\W_{(1),2,1})^\top\H_{t-1}\W_{(1),2,2}
=\begin{bmatrix}
\O_{d,s-1} & \lambda\w_{t-1}\\
\O_{s-d,s-1} & \O_{s-d,1}
               \end{bmatrix}. 
\end{align}
By setting $\W_{(1),2,3}=\O_{s,s}$ and 
\begin{align}
\B_{(1),2,3}=
\begin{bmatrix}
\I_d & \O_{d,s-d}
\end{bmatrix},
\end{align}
we have
\begin{align}
\newlsa_{\vtheta_{(1),2}}(\H_{t-1})
=\begin{bmatrix}
\O_{d,s-1} & \lambda\w_{t-1}
 \end{bmatrix}.
\end{align}

\item For $k=3$, we set $\B_{(1),3,1}=\B_{(1),3,3}=\O_{d,s}$.
We also set 
\begin{align}
\W_{(1),3,3}=\begin{bmatrix}
       \O_{n,s-n} &  \I_n \\
\O_{s-n,s-n} & \O_{s-n,n}
         \end{bmatrix},~
\W_{(1),3,1}=\begin{bmatrix}
\O_{n,s-n} & \O_{n,n}\\
\O_{n,s-n} &  \I_n\\
\O_{s-2n,s-n} &  \O_{s-2n,n}
         \end{bmatrix}.
\end{align}
We then have
\begin{align}
 \H_{t-1}\W_{(1),3,3}=\begin{bmatrix}
         \O_{d,s-n} & \X^\top
        \end{bmatrix},~
 \H_{t-1}\W_{(1),3,1}=\begin{bmatrix}
         \O_{d,s-n} & \Y_0^\top
        \end{bmatrix}.
\end{align}
By setting $\W_{(1),3,2}=\O_{s,s}$ and
\begin{align}
\B_{(1),3,2}=\begin{bmatrix}
          \O_{d,s-d} & -\I_d
         \end{bmatrix}, 
\end{align}
we obtain
\begin{align}
\newlsa_{\vtheta_{(1),3}}(\H_{t-1})=
(\X^\top\Y_0)\B_{(1),3,2}=
\begin{bmatrix}
                \O_{d,s-1} & -\X^\top\y
               \end{bmatrix}
\end{align}
by the definition of $\Y_0$.

\item For $k=4$, we set $\B_{(1),4,1}=\B_{(1),4,2}=\O_{d,s}$.
By choosing $\W_{(1),4,1}$ and $\W_{(1),4,2}$ for (\ref{eq:Htop-H-elsa})
according to Property \ref{lemma:mskmov}, 
we have
\begin{align}
(\H_{t-1}\W_{(1),4,1})^\top\H_{t-1}\W_{(1),4,2}=\begin{bmatrix}
          \O_{d,2n+d} & \eta\I_d & \O_{d,3}\\
          \O_{s-d,2n+d} & \O_{s-d,d} & \O_{s-d,3}\\
         \end{bmatrix} .
\end{align}
By setting $\W_{(1),4,3}=\O_{s,s}$ and $\B_{(1),4,3}=[-\I_d~~\O_{d,s-d}]$, we have
\begin{align}
\newlsa_{\vtheta_{(1),4}}(\H_{t-1})=
\begin{bmatrix}
\O_{d,2n+d} & -\eta\I_d & \O_{d,3}
\end{bmatrix}. 
\end{align}
\end{itemize}

After all, we have
\begin{align}
\P_{(1),t}:=\sum_{k=1}^4\newlsa_{\vtheta_{(1),k}}(\H_{t-1})
=\begin{bmatrix}
\O_{d,2n+d} & -\eta\I_d & \O_{d,2} & \Delta\w_{t-1}.
\end{bmatrix} 
\end{align}
for the first block, in which $\Delta\w_{t-1}$ is defined in
(\ref{eq:ridge-gd}).

We show a construction of the second block whose input is
$\P_{(1),t}$. For $k=1$, we set $\B_{(2),1,1}=\B_{(2),1,2}=\O_{d,s}$.
By choosing $\W_{(2),1,1}$ and $\W_{(2),1,2}$ for (\ref{eq:Htop-H-elsa})
according to Property \ref{lemma:mskmov}, we have
\begin{align}
(\P_{(1),t}\W_{(2),1,1})^\top\P_{(1),t}\W_{(2),1,2}=\begin{bmatrix}
          \O_{d,s-1} & -\eta\Delta\w_{t-1}\\
          \O_{s-d,s-1} & \O_{s-d,1}\\
         \end{bmatrix} .
\end{align}
By setting $\W_{(2),1,3}=\O_{s,s}$ and
$\B_{(2),1,3}=[\I_d~~\O_{d,s-d}]$, we have
\begin{align}
\newlsa_{\vtheta_{(2),1}}(\P_{(1),t})
=\begin{bmatrix}
                    \O_{d,s-1} & -\eta\Delta\w_{t-1}
                   \end{bmatrix}.
\end{align}
By setting $\W_{(2),k,l}=\O_{s,s}$ and 
$\B_{(2),k,l}=\O_{d,s}$ for $k=2,3,4$ and $l=1,2,3$, we have
\begin{align}
\P_{(2),t}=\begin{bmatrix}
                    \O_{d,s-1} & -\eta\Delta\w_{t-1}
                   \end{bmatrix}.
\end{align}

As a result, $\H_t:=\module_t(\H_{t-1})$ is obtained by
\begin{align}
\H_t=\P_{(2),t}+\H_{t-1}
=\begin{bmatrix}
\X^\top & \Y_0^\top & \lambda\I_d & \sqrt{\eta}\I_d & \u & \O_{d,1} & \w_t
   \end{bmatrix},
\end{align}
where $\w_t$ is (\ref{eq:ridge-gd}) as desired.

By successively applying $\module_t(\H_{t-1})$ for $t=1,\ldots,T$, we
have 
\begin{align}
\H_T:=\module_t(\H_{T-1})
=\begin{bmatrix}
\X^\top & \Y_0^\top & \lambda\I_d & \sqrt{\eta}\I_d & \u & \O_{d,1} & \w_T
   \end{bmatrix}.
\end{align}

We then construct the output module by $\module_{T+1}$.
We write $(\W_{(j),k,l},\B_{(j),k,l})=(\W_{(j),T+1,k,l},\B_{(j),T+1,k,l})$ below.
For $k=1$, we set
$\B_{(1),1,1}=\B_{(1),1,2}=\O_{d,s}$.  By choosing $\W_{(1),1,1}$ and
$\W_{(1),1,2}$ for (\ref{eq:Htop-H-elsa}) 
according to Property \ref{lemma:mskmov}, we have
\begin{align}
\W_{(1),1,1}^\top\H_T^\top\H_T\W_{(1),1,2}=\begin{bmatrix}
          \O_{1,s-2} & \u^\top\w_T & 0\\
          \O_{s-1,s-2} & \O_{s-1,1} & \O_{s-1,1}\\
         \end{bmatrix} .
\end{align}
By setting $\W_{(1),1,3}=\O_{s,s}$ and 
\begin{align}
\B_{(1),1,3}=\begin{bmatrix}
1 & \O_{1,s-1}\\
\O_{d-1,1} & \O_{d-1,s-1}
\end{bmatrix},
\end{align}
we have 
\begin{align}
\newlsa_{\vtheta_{(1),1}}(\H_T)
=\begin{bmatrix}
         \O_{1,s-2} & \u^\top\w_T & 0\\
          \O_{d-1,s-2} & \O_{d-1,1} & \O_{d-1,1}\\
 \end{bmatrix}.
\end{align}
By setting $\W_{(1),k,l}=\O_{s,s}$ and $\B_{(1),k,l}=\O_{d,s}$ for
$k=2,3,4$ and $l=1,2,3$, we have
\begin{align}
\P_{(1),T+1}=\begin{bmatrix}
         \O_{1,s-2} & \u^\top\w_T & 0\\
          \O_{d-1,s-2} & \O_{d-1,1} & \O_{d-1,1}\\
 \end{bmatrix}.
\end{align}
By choosing $\vtheta_{(2),T+1,1}$ according to Property \ref{lemma:elsa-skip},
we have
\begin{align}
\newlsa_{\vtheta_{(2),T+1,1}}(\P_{(1),T+1})=\P_{(1),T+1}.
\end{align}
By setting $\W_{(2),T+1,k,l}=\O_{s,s}$ and $\B_{(2),T+1,k,l}=\O_{d,s}$ for
$k=2,3,4$ and $l=1,2,3$, we have
\begin{align}
\P_{(2),T+1}=\P_{(1),T+1}=
\begin{bmatrix}
         \O_{1,s-2} & \u^\top\w_T & 0\\
          \O_{d-1,s-2} & \O_{d-1,1} & \O_{d-1,1}\\
 \end{bmatrix}.
\end{align}
We then obtain
\begin{align}
\H_{T+1}&:=\P_{(2),T+1}+\H_T\notag\\
&=\begin{bmatrix}
\X^\top & \Y_0^\top & \lambda\I_d & \sqrt{\eta}\I_d & \u & \z & \w_T
   \end{bmatrix},
\end{align}
where $\z=[\u^\top\w_T~~\O_{1,d-1}]^\top$ which includes a model
prediction for $\u$ after $T$ steps.

We have several remarks.
\begin{itemize}

\item Although some components are obviously redundant, it is necessary
      for having a common structure.

 \item A module defined by (\ref{eq:mod-elsa}) adapts to
(\ref{eq:designed-input}) since the first block can implement LSA for
(\ref{eq:designed-input}) and the second block can be a skip connection
by Property \ref{lemma:elsa-skip}.  In this way, ELSA gives us flexible
matrix computation.

\item We used two sequential multi-head ELSAs in a
  module. This allows us to perform flexible calculations using a skip
  connection, multiplications with two or more matrices.  The number
  of multi-head ELSAs required depends on the task. 
  It is important that various matrix multiplication operations can be
  implemented by using a sufficient number of multi-head ELSAs.  In
  applications, it is natural that we can choose a model complexity
  under an input form, and it may not be natural that we need to explore
  an input form under a fixed architecture.

\item Although we have checked the justification of these matrix
      computations yb computer, it is not clear that these
      implementations are obtained by training; i.e. example-based
      updating of the weight and bias matrices.
Obviously, the above implementations of the gradient descent for ridge
      regression are heuristically constructed. For example, there may
      be implementations using LSA under a different input form or
      module; e.g. using mutiple LSAs for a single step. The
      same argument applies to ELSA.
These implementations may
      be found through training.
In addition, there may be an advantage of nonlinearity in self-attention.
 These are parts of our future work.

\item In this example of ridge regression, we extract relevant
       submatrices using the mask and move operation.
This extraction is a hard extraction in the sense that irrelevant
       elements are set to zero. However, depending on the task,
       weighted extraction is possible; i.e. it can be called a
       soft extraction.
In fact, an attention mechanism in a transformer is used in this way
      when used as a language model.

\end{itemize}

\section{Conclusions and future works}

The attention mechanism plays a key role in the in-context learning
ability of transformers.  In the attention mechanism, an attention
matrix encodes relationships among words in a sentence and is used as
weights for words in a sentence. Although the attention mechanism is
effective in language models, it is questionable whether it is suitable
for in-context learning in general tasks. In fact, we may need an
appropriate design of an input form (prompt) for a suitable
implementation of an algorithm, since matrix multiplication implemented
by an attention layer is restrictive. In this paper, by introducing a
bias matrix term in addition to multiplication of a weight matrix and an
input, we extended linear self-attention to cover various matrix
computations such as a constant matrix output, a skip connection, and
a multiplication of two matrices. As an example, we
heuristically implemented a batch-type gradient descent algorithm for
ridge regression by using a naive linear self-attention under a designed
input form and the extended linear self-attention under an input formed
naturally by enumerating variable vectors or matrices. Note that the
extended linear self-attention can also adapt to an input form, under
which a naive linear self-attention is applied.  In applications, the
training process when using the extended linear self-attention is not
clear. Therefore, as a future work, we need to numerically analyze the
extended linear self-attention. We also need to test the extended linear
self-attention with nonlinearity and consider the implementation of
various algorithms by using it.

\appendix

\section{Implementation of Gaussian elimination by using computational components}

\subsection{Background}

We here construct a solver for a closed-form solution of a ridge
estimate by using standard components of layered neural network.

\cite{Akyurek2022} has also implemented a closed-form solution of ridge
regression using a transformer. The important point in
\cite{Akyurek2022} is that several computational primitives that are
realized by a transformer are given explicitly, and a solver of ridge
regression is expressed as a processing or operation using the
primitives.  In other words, it implements an algorithm that outputs
ridge estimates by using computational primitives. In the implementation
of a solver of ridge regression, we need multiplication and division
operations. In \cite{Akyurek2022}, multiplication operations are required
in the implementation and are realized by using the GeLU (Gaussian
Error Linear Unit) nonlinearity based on Taylor's expansion; see also
\cite{Liang2017}. On the other hand, division operations are essential since we need
the inverse of the matrix to obtain ridge estimates. Note that we can say that
a gradient descent approximately computes the matrix inversion only by
multiplications and additions of matrices. However, we need an infinite
number of iterations if we obtain an exact solution. \cite{Akyurek2022}
uses the Sherman–Morrison formula to reduce the matrix inversion to
a sequence of rank-one updates performed example-by-example and, then,
the division operation is implemented by the ``LayerNorm'' function.

A ridge estimate is obtained by solving a system of linear equations. In
this appendix, we implement Gaussian elimination to solve a
system of linear equations by using computational components. 
Gaussian elimination is a basic method and is well known in numerical
computing; e.g. see \cite{Gentle1998,Michailidis2011}. Apart from
transformers, we just implement it by using a combination of standard
components of layered neural networks and it is not along with neural
network implementations. Nevertheless, it tells us what components we
need to realize an algorithm. Our implementation differs from
\cite{Akyurek2022} in particular with respect to multiplication and division
operations. We here implement multiplication operations by matrix multiplications
and a division operation by a ReLU network with one hidden layer.

\subsection{Some network components}

\subsubsection{Some notations used in appendix}

$\|\cdot\|$ denotes the Euclidean norm.  $\odot$ denotes the Hadamard
product.  $\relu$ is a ReLU function; i.e. $\relu(x)=(x)_+$ for
$x\in\R$.  $\idrelu(x)$ is an identity function implemented by ReLU;
i.e. $\idrelu(x)=\relu(x)-\relu(-x)$.  
$\delta_{i,j}$ is the Kronecker delta, which means $\delta_{i,j}=1$ if
$i=j$ and $0$ if $i\neq j$.

\subsubsection{Network with one hidden layer}

Let $\X$ and $\Z$ be $m\times n$ matrices. We consider to construct a
mapping from $\X$ into $\Z$ by a network with one hidden layer.
More precisely, the output of the network is defined by
\begin{align}
\label{eq:network-comp}
\Z=\Z(\X):=\sum_{k=1}^{\ok}
\left\{\V_k\odot\sigma\left(\W_k\odot\X+\B_k\right)+\C_k\right\},
\end{align}
where $\W_k$, $\V_k$, $\B_k$ and $\C_k$ are $m\times n$ matrices and
$\sigma$ is a nonlinear activation function which is componentwisely
applied if it is applied to a matrix.  We call this formulation of a
layered neural network a network component.
Although the input to the network may not be a usual weighted sum,
$\W_k$, $\V_k$, $\B_k$ and $\C_k$ correspond to input weights, output
weights, input biases and output biases. This mapping is componentwisely
written by
\begin{align}
\Z[i,j]=\sum_{k=1}^{\ok}\left\{
\V_k[i,j]\sigma(\W_k[i,j]\X[i,j]+\B_k[i,j])+\C_k[i,j]\right\}.
\end{align}

As an example, we here implement a componentwise affine transform by
using this network component.  Let $\sigma$ be a ReLU; i.e.
$\sigma=\relu$.  We assume that $\ok>2$.
We also assume that $\gamma_{i,j}\in\R$ and $C_{i,j}\in\R$.  By setting
$\W_1[i,j]=1$, $\W_2[i,j]=-1$, $\V_1[i,j]=\gamma_{i,j}$, $\V_2[i,j]=-\gamma_{i,j}$,
$\V_k[i,j]=0$ for $k=3,\ldots,\ok$, $\B_k[i,j]=0$ for $k=1,\ldots,\ok$,
$\C_1[i,j]=C_{i,j}$ and $\C_k[i,j]=0$ for $k=2,\ldots,\ok$, we have
\begin{align}
\label{eq:construct-affine}
\Z[i,j]&=\gamma_{i,j}\{\relu(\X[i,j])-\relu(-\X[i,j])\}+C_{i,j}\notag\\
&=\gamma_{i,j}\idrelu\left(\X[i,j]\right)+C_{i,j}\notag\\
&=\gamma_{i,j}\X[i,j]+C_{i,j}.
\end{align}
Therefore, a network component implements a componentwise affine transform. 
If $\gamma_{i,j}=1$ and $C_{i,j}=0$ then $\Z[i,j]=\X[i,j]$ which is an
identity function. If $\gamma_{i,j}=0$ then $\Z[i,j]=C_{i,j}$ which is an
arbitrary constant.
If $\gamma_{i,j}=1$ then $\Z[i,j]=\X[i,j]+C_{i,j}$ which is the addition of a constant.

On the other hand, we define $\M_{i,j,k,l}^{m,n}$ as an $m\times n$
matrix whose $(s,t)$-entry is $1$ if $1\le i\le s\le j\le m,~1\le k\le
t\le l\le n$ and $0$ otherwise. We call this a mask.  Conversely, we
define $\oM_{i,j,k,l}^{m,n}$ as an $m\times n$ matrix whose
$(s,t)$-entry is $0$ if $1\le i\le s\le j\le m,~1\le k\le t\le l\le n$
and $1$ otherwise. We call this an anti-mask.  For example,
$\M_{i,j,k,l}^{m,n}\odot\A$ implements a masking operation for
$\A[i\:j,k\:l]$, in which the size of the resulting matrix is $m\times
n$ and elements except for $\A[i\:j,k\:l]$ are zeros.  In
(\ref{eq:network-comp}), we set that $\ok=1$, $\B_1=\C_1=\O_{m,n}$,
$\V_1=\M_{i,j,k,l}^{m,n}$, any elements of $\W_1$ is $1$ and $\sigma$ is
an identity function. Then, we have
$\Z=\Z(\X)=\M_{i,j,k,l}^{m,n}\odot\X$; i.e. it can implement a masking
operation. Similarly, by using $\oM_{i,j,k,l}^{m,n}$ instead of $\M_{i,j,k,l}^{m,n}$, 
(\ref{eq:network-comp}) can implement an
anti-mask operation.

\subsubsection{Skip connection}

We consider two types of skip connection.  Let $\A$ be an $m\times n$
matrix and $\module$ is a module that consists of a set of components.
$\module$ receives $\A$ and outputs a matrix with a certain size. Addition
type skip connection computes $\module(\A)+\gamma\A$, in which the size
of $\module(\A)$ is $m\times n$ and $\gamma\in\{-1,+1\}$. Thus, it allows
subtraction.  Multiplication type skip connection computes $\gamma
\module(\A)\A$ or $\gamma \A\module(\A)$, in which the size of $\module(\A)$ is
$(s,m)$ for the former and is $(n,s)$ for the latter for a certain $s$.
Here, again we allow to assign $\gamma\in\{-1,+1\}$ while it may be
implemented by a part of computation in $\module$.

\subsubsection{Weighted sum operation}

We define $\SUM$ as an operation such that $\SUM(\A)$ for $\A$ is the
sum of all entries of $\A$. Let $\W$ and $\A$ be $m\times n$ weight and
input matrices respectively.  Then, $\SUM(\W\odot\A)$ is a weighted sum
and, moreover, $\SUM(\W\odot\A)+b$ is a weighted sum with a bias $b\in\R$.
Indeed, the weighted sum operation is relatively versatile.

We define a mapping from an $m_i\times n_i$ matrix into an $m_o\times
n_o$ matrix using the weighted sum with a bias term.  Let $\A$ be an
$m_i\times n_i$ matrix and $\P$ be an $m_o\times n_o$
matrix. $(m_o,n_o)$ is arbitrarily chosen as long as the operation can
be defined.  We define $\P[s,t]=\SUM(\W_{s,t}\odot\A)+b_{s,t}$.  If we
set $\W_{s,t}=\gamma_{s,t}\M_{i',i',k',k'}^{m,n}$ for
$\gamma_{s,t}\in\R$ then $\P[s,t]=\gamma_{s,t}\A[i',k']+b_{s,t}$; i.e.
affine transformation of $\A[i',k']$.  Therefore, if we set
$\gamma_{s,t}=1$ and $b_{s,t}=0$ then $\P[s,t]=\A[i',k']$, which enables
us an arbitrary mask and move operation that is also able to be
implemented by matrix multiplications with weight matrices as shown in
this paper. And, if we set $\W_{s,t}=\O_{m,n}$ then $\P[s,t]=b_{s,t}$;
i.e. inserting a constant.

For example, we implement a mask and move operation of $\A[i\:j,k\:l]$,
in which the size of $\P$ is the same as that of $\A$. For
$s=1,\ldots,m$ and $t=1,\ldots,n$, we define
$\W_{s,t}:=\M_{i',i',k',k'}^{m,n}$ for $1\le i'\le m$ and $1\le k'\le
n$. We assume that all indices are within a matrix size.  We define
$S:=\{(i'+a,k'+b):i\le i'\le j,~k\le k'\le l\}$ and set
\begin{align}
\W_{s,t}=\begin{cases}
\M_{i',i',k',k'}^{m,n} & (s,t)\in S\\
\O_{m,n} & {\rm otherwise}
         \end{cases}.
\end{align}
Then, it is obvious that 
\begin{align}
\P[s,t]=\SUM(\W_{s,t}\odot\A)=
\begin{cases}
\A[i',k'] & (s,t)\in S\\
0 & {\rm otherwise}
\end{cases}
\end{align}
holds. 

On the other hand, in constructing a network with one hidden layer,
we define
\begin{align}
\label{eq:network-ws}
\Z[i,j]:=\sum_{k=1}^{\ok}\left\{
\V_k[i,j]\sigma(\SUM(\W_{k,i,j}\odot\X)+\B_k[i,j])+\C_k[i,j]\right\},
\end{align}
where $\W_{k,i,j}$ is an $n\times m$ matrix and the other symbols are
the same as in the previous subsection except that the size of $\Z$ is
arbitrary and $(i,j)$ is within the size of $\Z$.  For example,
(\ref{eq:network-comp}) can be represented using the weighted sum with a
mask, in which we set $\W_{k,i,j}=\W_k\odot\M_{i,i,j,j}^{m,n}$ in
(\ref{eq:network-ws}).  If we set $\M_{i',i',j',j'}^{m,n}$ instead of
$\M_{i,i,j,j}^{m,n}$ then $\Z[i,j]$ is a function of $\X[i',j']$ which
is an arbitrary position in $\X$.

\subsection{Approximation of division by ReLU network}
\label{sec:invapprox}

We here give a realization of division operator by using a network
component.

To do this, we consider an approximation of a function $f$ on $\R$ by a
weighted sum of ReLUs; i.e. a ReLU network with one hidden layer.  We
assume that $f(x)=f(-x)$, $|f(x)|$ is monotonically decreasing and
$\lim_{x\to\infty}|f(x)|=0$ for $x\in\R$. We set $f(x)=1/x^2$ for
$x\in\R$. Then, $f$ satisfies the above condition and we can approximate
$1/x=xf(x)$ by a network with ReLUs and a skip connection for $x$ if we
can approximate $f$ by a network with ReLUs. Thus, division operator is
approximately realized by a ReLU network.

We define $(x)_+=\max\{0,x\}$ which is ReLU.

Let $x_0<x_1<\ldots<x_n<x_{n+1}$ be positive real numbers, where we set
$x_0=0$.  We define $I_k^+=(x_{k-1},x_k]$ and $I_k^-=(-x_k,-x_{k-1}]$
for
$k=1,\ldots,n+1$. $\{I_{n+1}^-,\ldots,I_1^-,I_1^+,\ldots,I_{n+1}^+\}$ is
a partition of $(-x_{n+1},x_{n+1}]$.  We set $y_k=f(x_k)=f(-x_k)$ for
$k=1,\ldots,n$ and define $y_0=y_1$ and $y_{n+1}=0$.  In the following,
we construct a piecewise linear function that takes $y_k$ at $\pm x_k$ and
approximates $f$ on $[-x_{n+1},-x_1]\bigcup[x_1,x_{n+1}]$.  More
specifically, we consider to approximate $f$ by piecewise linear
functions which are constructed by a sum of hard sigmoids that are
implemented by a sum of two ReLUs. 
It is roughly shown in Figure \ref{fig:approx-f}.

Note that this approximation may
not be mathematically rigorous and this point is discussed later.

\newcommand{\figwidtha}{30mm}
\newcommand{\figwidthb}{45mm}
\begin{figure}[h]
\begin{center}
\begin{minipage}[t]{\figwidtha}
\begin{center}
\includegraphics[width=\figwidtha]{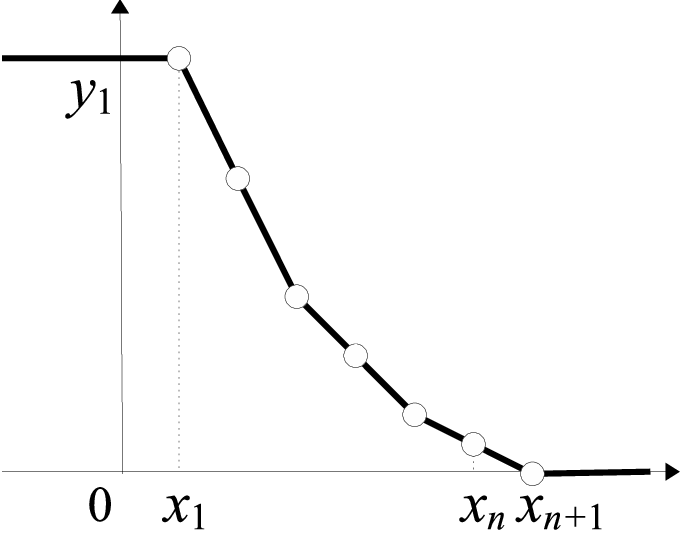}\\
(a)
\end{center}
\end{minipage}\hspace{3mm}
\begin{minipage}[t]{\figwidtha}
\begin{center}
\includegraphics[width=\figwidtha]{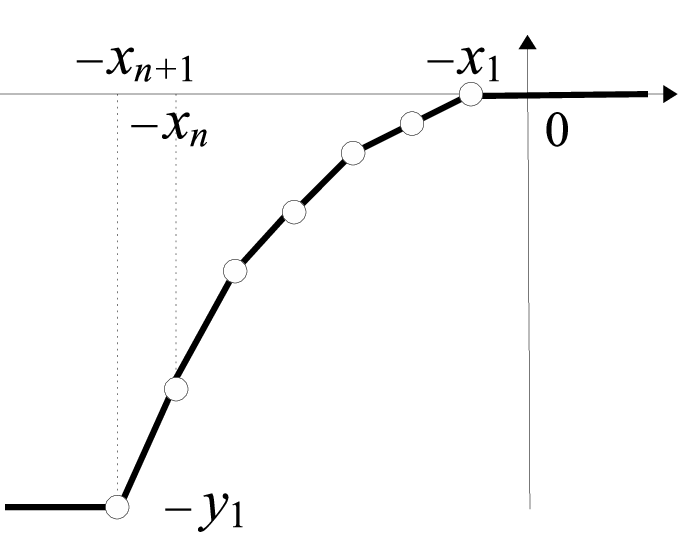}\\
(b)
\end{center}
\end{minipage}\hspace{3mm}
\begin{minipage}[t]{\figwidthb}
\begin{center}
\includegraphics[width=\figwidthb]{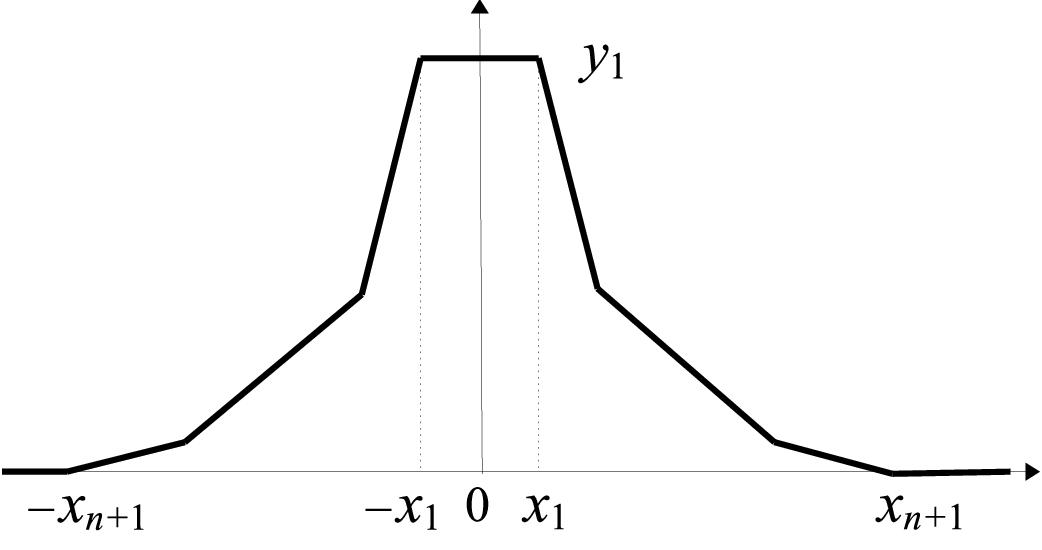}\\
(c)
\end{center}
\end{minipage}

\end{center}
\caption{Approximation $f$ by sum of ReLUs (hard sigmoids). (a) Approximation by
$r^+$. (b) Approximation by $r^-$. (c) $r^++r^-$ for approximating $f$.}
\label{fig:approx-f}
\end{figure}

\subsubsection{Construction of approximator}

We define, for $x\in\R$ and $k=2,\ldots,n+1$,
\begin{align}
\of_{k,0}^+(x)&=\alpha_k^+(x-x_k)\\ 
\of_{k,-1}^+(x)&=\alpha_k^+(x-x_{k-1}),
\end{align}
where
\begin{align}
\alpha_k^+=\frac{y_k-y_{k-1}}{x_k-x_{k-1}}. 
\end{align}
We then define
\begin{align}
r_k^+(x)=r_{k,0}^+(x)+r_{k,-1}^+(x),
\end{align}
where
\begin{align}
r_{k,0}^+(x)&=(\of_{k,0}^+(x))_+\\
r_{k,-1}^+(x)&=-(\of_{k,-1}^+(x))_+.
\end{align}
Note that $\alpha_k^+<0$ holds since $y_k<y_{k-1}$ and $x_k>x_{k-1}$.
Then, it is easy to see that
\begin{align}
r_k^+(x)=
\begin{cases}
y_{k-1}-y_k & x\le x_{k-1}\\
\alpha_k^+(x-x_k) & x_{k-1}<x\le x_k\\
0 & x_k<x.
\end{cases} 
\end{align}
Therefore $r_k^+$ is a hard sigmoid.

We define
\begin{align}
r^+(x)=\sum_{k=2}^{n+1}r_k^+(x). 
\end{align}
\begin{itemize}
 \item If $x\le x_1$ then $x<x_k$ holds for any
$k=2,\ldots,n+1$. Therefore, we have $r_k^+(x)=y_{k-1}-y_k$ for any $k$.
We thus have, for $x\le x_1$, 
\begin{align}
r^+(x)=(y_1-y_2)+(y_2-y_3)+\cdots+(y_n-y_{n+1})=y_1-y_{n+1}=y_1
\end{align}
by the definition of $y_{n+1}$.
\item If $x_{n+1}<x$ then $x_k<x$ for $k=2,\ldots,n+1$.
Therefore, $r_k^+(x)=0$ for $k=2,\ldots,n+1$. This implies that
$r^+(x)=0$ for $x>x_{n+1}$. 
\item If $x_1< x\le x_{n+1}$ then there exists $m\in\{2,\ldots,n+1\}$ such that 
$x_{m-1}<x\le -x_m$. 
We then have
\begin{align}
r^+(x)&=\sum_{k=2}^{m-1}r_k^+(x)+r_m^+(x)
+\sum_{k=m+1}^{n+1}r_k^+(x)\notag\\
&=0+\alpha_m(x-x_m)+
(y_m-y_{m+1})+\cdots+(y_n-y_{n+1})\notag\\
&=\alpha_m(x-x_m)+y_m-y_{n+1}\notag\\
&=\alpha_m(x-x_m)+y_m.
\end{align}
\end{itemize}
As a result, we have
\begin{align}
r^+(x)=
\begin{cases}
y_1 & x\le x_1\\
\alpha_k^+(x-x_k) +y_k & x\in I_k^+\\
0 & x>x_{n+1},
\end{cases} 
\end{align}
where $k=2,\ldots,n+1$.
Note that we have, for $k=2,\ldots,n+1$
\begin{align}
r^+(x_k)&=\alpha_k^+
(x_k-x_k)+y_k=y_k\\
r^+(x_{k-1})&=\alpha_k^+
(x_{k-1}-x_k)+y_k
=(y_{k-1}-y_k)+y_k=y_{k-1}
\end{align}
which implies that $r^+(x)$ is consistent with $f(x)$ at $x=x_k$ for $k=2,\ldots,n+1$.

On the other hand, we define, for $x\in\R$ and $k=2,\ldots,n+1$,
\begin{align}
\of_{k,0}^-(x)&=\alpha_k^-(x+x_k)\\ 
\of_{k,-1}^-(x)&=\alpha_k^-(x+x_{k-1}),
\end{align}
where
\begin{align}
\alpha_k^-=-\frac{y_{k-1}-y_k}{-x_{k-1}-(-x_k)}
=\frac{y_k-y_{k-1}}{x_k-x_{k-1}}.
\end{align}
We then define
\begin{align}
r_k^-(x)=r_{k,0}^-(x)+r_{k,-1}^-(x),
\end{align}
where
\begin{align}
r_{k,0}^-(x)&=(\of_{k,0}^-(x))_+\\
r_{k,-1}^-(x)&=-(\of_{k,-1}^-(x))_+.
\end{align}
Note that $\alpha_k^-<0$ holds since $y_k<y_{k-1}$ and $x_k>x_{k-1}$.
Then, it is easy to see that
\begin{align}
r_k^-(x)=
\begin{cases}
y_k-y_{k-1} & x\le -x_k\\
-\alpha_k^-(x+x_{k-1}) & -x_k<x\le -x_{k-1}\\
0 & -x_{k-1}<x.
\end{cases} 
\end{align}
We define
\begin{align}
r^-(x)=\sum_{k=2}^{n+1}r_k^-(x). 
\end{align}
As in case of $r^+$, we have
\begin{align}
r^-(x)=
\begin{cases}
-y_1 & x\le -x_{n+1}\\
-\alpha_k^-(x+x_{k-1})+y_{k-1}-y_1 & x\in I_k^-\\
0 & x>-x_1,
\end{cases} 
\end{align}
where $k=2,\ldots,n+1$.
Note that we have, for $k=2,\ldots,n+1$
\begin{align}
r^-(-x_k)&=-\alpha_m^-
(-x_k+x_{k-1})+y_{k-1}-y_1=y_k-y_1\\
r^-(-x_{k-1})&=-\alpha_m^-
(-x_{k-1}+x_{k-1})+y_{k-1}-y_1=y_{k-1}-y_1
\end{align}
which implies that $r^-(x)$ is consistent with $f(x)-y_1$ at $x=-x_k$ for
$k=2,\ldots,n+1$.

Finally, we define
\begin{align}
\sigma_f(x)=r^+(x)+r^-(x).
\end{align}
By the construction of $r^+$ and $r^-$, we have
\begin{align}
\label{eq:invsqr}
\sigma_f(x)=
\begin{cases}
0 & x>x_{n+1}\\
\alpha_k^+(x-x_k)+y_k & x\in I_k^+\\
y_1 & x\in I_1^+\bigcup I_1^-\\
-\alpha_k^-(x+x_{k-1})+y_{k-1} & x\in I_k^-\\
0 & x\le -x_{n+1},
\end{cases}
\end{align}
where $k=2,\ldots,n+1$.
Obviously, we have $\sigma_f(x_k)=\sigma_f(-x_k)=y_k$ as desired.

If $f(x)=1/x^2$ then we write $\invsqr(x)=\sigma_f(x)$. 

\subsubsection{Discussion for approximation quality}

In the above construction, the piecewise linear approximation may work
well on $[-x_{n+1},-x_1]\cup[x_1,x_{n+1}]$ if $n$ is large.  On the
other hand, it causes a problem in $I_0=[-x_1,x_1]$ and
$I_{\infty}=(-\infty,-x_{n+1}]\cup[x_{n+1},\infty)$.  The approximation
error on $I_{\infty}$ may not so critical since $|f|$ is monotonically
decreasing and we can take sufficiently large value for $x_{n+1}$.  In
$I_0$, the difference is huge when $x\to 0$. However, in case of
$x\simeq 0$, division cannot be properly conducted in computer due to
overflow. Therefore, we can say that, by choosing small value for $x_1$,
our approximation is not mathematically rigorous but computationally
feasible.

\subsection{Combination of computational components that solves
Gaussian elimination}

\subsubsection{Preliminary}

Let $\F$ be an $m\times m$ matrix and $\valpha$ be an $m\times 1$
vector.  For an $m\times 1$ vector $\x$, we consider to solve a set of
linear equations $\F\x=\valpha$ in terms of $\x$. There is an unique
solution if $\F$ is invertible.  We assume that $\F$ is non-singular.
Gaussian elimination consists of forward elimination and backward
substitution.  The forward elimination algorithm converts $\F$ to an
upper triangular matrix, by which the solution is obtained backward
substitution. We call a set of linear equations determined by
$[\F,\valpha]$ a system of linear equations. If $\F$ is an upper
triangular matrix then we call it an upper triangular system of linear
equations. We here construct modules that implement Gaussian
elimination. 

Withou loss of generality, for $m\ge 2$, let $\P$ be an $m\times(m+1)$
matrix which represents a system of linear equations.  The $(i,j)$-entry
of $\P$ is denoted by $p_{i,j}$ instead of $\P[i,j]$ for simplifying the
expression. We first construct a module that executes forward
elimination; i.e. column reduction process. $\P$ is an input to
this module and the output is an upper triangular system
of linear equations. To do this, we set the input matrix to
\begin{align}
\P_1:=\begin{bmatrix}
\P\\
\O_{1,m+1}       
      \end{bmatrix},
\end{align}
in which an extra zero vector is added to the last row of $\P$.
$\P_1$ is an $(m+1)\times(m+1)$ matrix.

\subsubsection{Elimination of the first column}

We demonstrate an elimination process of the first column.  

\begin{itemize}
\item By a masking operation for $\P_1$, we obtain an $(m+1)\times
      (m+1)$ matrix $\Z^{(1)}:=\Z^{(1)}(\P_1)$, in which
\begin{align}
\Z^{(1)}[i,j]=\begin{cases}
               p_{i,j} & i=j=1\\
0 & {\rm otherwise}
              \end{cases}.
\end{align}

\item According to \ref{sec:invapprox}, by using a network component,
we compute an $(m+1)\times(m+1)$
      matrix $\Z^{(2)}$ whose $(i,j)$-entry is given by
\begin{align}
\Z^{(2)}[i,j]:=
\begin{cases}
\invsqr\left(\Z^{(1)}[i,j]\right)\simeq 1/\Z^{(1)}[i,j]^2\simeq
1/p_{i,j}^2  & i=j=1\\
\idrelu\left(\Z^{(1)}[i,j]\right)=\Z^{(1)}[i,j]=0 & {\rm otherwise}
\end{cases},
\end{align}
where $\invsqr$ is defined in \ref{sec:invapprox} and $\idrelu$ is an
identity function implemented by ReLU. Hereafter,
      we use ``$=$'' instead of ``$\simeq$''.

\item By receiving $\Z^{(1)}$ from a skip connection, we obtain
$\Z^{(3)}:=-\Z^{(1)}\Z^{(2)}$, 
which is an $m\times m$ matrix whose $(1,1)$-entry is approximately
      $-1/p_{1,1}$ and otherwise $0$.

\item By a masking operation for $\P_1$, we obtain an
      $(m+1)\times (m+1)$ matrix $\Z^{(4)}:=\Z^{(4)}(\P_1)$, in which
\begin{align}
\Z^{(4)}[i,j]=
\begin{cases}
p_{i,j}  & j=1,~i=2,\ldots,m\\
0 & {\rm otherwise}
\end{cases},
\end{align}
where $\P_1$ comes from a skip connection.

\item We obtain $\Z^{(5)}:=\Z^{(4)}\Z^{(3)}$, which is given by
\begin{align}
\Z^{(5)}=\begin{bmatrix}
    \O_{1,1} & \O_{1,m}\\
    -\gamma_2 & \O_{1,m}\\
    \cdots & \cdots \\
    -\gamma_m & \O_{1,m}\\
    \O_{1,1} & \O_{1,m}\\
  \end{bmatrix}
\end{align}
where $\gamma_i:=p_{i,1}/p_{1,1}$.

\item According to (\ref{eq:construct-affine}), we obtain
      $\Z^{(6)}=\Z^{(5)}+\I_{m+1}$ by using a network component.

\item By receiving $\P_1$ from a skip connection, we obtain
$\P_2:=\Z^{(6)}\P_1$, which is given by 
\begin{align}
\P_2=
\begin{bmatrix}
    p_{1,1} & p_{1,2} & p_{1,3} & \cdots & p_{1,m} & p_{1,m+1}\\
    0 & p_{2,2}^{(2)} & p_{2,3}^{(2)} & \cdots & 
p_{2,m}^{(2)} & p_{2,m+1}^{(2)}\\
\cdots & \cdots & \cdots & \cdots & \cdots & \cdots\\
    0 & p_{m,2}^{(2)} & p_{m,3}^{(2)} & \cdots & 
p_{m,m}^{(2)} & p_{m,m+1}^{(2)}\\
0 & 0 & 0 & \cdots & 0 & 0
   \end{bmatrix},
\end{align}
where
\begin{align}
p_{i,j}^{(2)}&=p_{i,j}-\gamma_ip_{1,j}
\end{align}
for $i=2,\ldots,m$ and $j=2,\ldots,m+1$.

\end{itemize}

\subsubsection{Elimination of the $k$-th column}

Elimination of the $k$-th column is almost a copy of the process for the first column.
Since there may be no confusion, we employ the same symbols for matrices as
in the previous subsection. 

At the $k$-th step, we have
\begin{align}
\P_{k-1}=
\begin{bmatrix}
    p_{1,1} & \cdots & p_{1,k-1} & p_{1,k} & p_{1,k+1} & 
\cdots &  p_{1,m+1}\\
               0 & \cdots & p_{2,k-1}^{(2)} & p_{2,k}^{(2)} & p_{2,k+1}^{(2)} & 
\cdots & p_{2,m+1}^{(2)}\\
\cdots & \cdots & \cdots & \cdots & \cdots & \cdots & \cdots\\
               0 & \cdots & p_{k-1,k-1}^{(k-1)} & p_{k-1,k}^{(k-1)} 
& p_{k-1,k+1}^{(k-1)} & \cdots & p_{k-1,m+1}^{(k-1)}\\
               0 & \cdots & 0 & p_{k,k}^{(k-1)} & p_{k,k+1}^{(k-1)} & 
\cdots & p_{k,m+1}^{(k-1)}\\
               0 & \cdots & 0 & p_{k+1,k}^{(k-1)} & p_{k+1,k+1}^{(k-1)} & 
\cdots &  p_{k+1,m+1}^{(k-1)}\\
\cdots & \cdots \cdots & \cdots & \cdots & \cdots & \cdots\\
    0 & \cdots & 0 & p_{m,k}^{(k-1)} & p_{m,k+1}^{(k-1)} & 
\cdots &  p_{m,m+1}^{(k-1)}\\
    0 & \cdots & 0 & 0 & 0 & \cdots & 0
   \end{bmatrix}.
\end{align}

The construction of the $k$-th step is as follows.
\begin{itemize}
\item By a masking operation for $\P_{k-1}$, we obtain an
      $(m+1)\times (m+1)$ matrix
$\Z^{(1)}=\Z^{(1)}(\P_{k-1})$, in which
\begin{align}
\Z^{(1)}[i,j]=
\begin{cases}
p_{i,j}^{(k-1)} & i=j=k\\
0 & {\rm otherwise}
\end{cases}
\end{align}
for $1\le k\le m$.

\item According to \ref{sec:invapprox}, by using a network component, we
      obtain an $(m+1)\times(m+1)$ matrix $\Z^{(2)}$ whose
      $(i,j)$-entry is given by
\begin{align}
\Z^{(2)}[i,j]=
\begin{cases}
\invsqr\left(\Z^{(1)}[i,j]\right)
\simeq 1/(\Z^{(1)}[k,k])^2=1/(p_{k,k}^{(k-1)})^2
 & i=j=k\\
\idrelu\left(\Z^{(1)}[i,j]\right)=\Z^{(1)}[i,j]=0 & {\rm otherwise}
\end{cases}.
\end{align}

\item By receiving $\Z^{(1)}$ from a skip connection, we obtain
$\Z^{(3)}:=-\Z^{(1)}\Z^{(2)}$,
which is a matrix whose $(k,k)$-entry is approximately
      $-1/p_{k,k}^{(k-1)}$ and otherwise $0$.

\item By a masking operation for $\P_{k-1}$, we obtain an
      $(m+1)\times (m+1)$ matrix
$\Z^{(4)}:=\Z^{(4)}(\P_{k-1})$, in which
\begin{align}
\Z^{(4)}[i,j]:=
\begin{cases}
p_{i,j}^{(k-1)}  & j=k,~i=k+1,\ldots,m\\
0 & {\rm otherwise}
\end{cases},
\end{align}
where $\P_{k-1}$ comes from a skip connection.

\item We obtain
$\Z^{(5)}:=\Z^{(4)}\Z^{(3)}$, in which $(i,j)$-entry is 
\begin{align}
\Z^{(5)}[i,j]=
\begin{cases}
-\gamma_i & j=k,~i=k+1,\ldots,m\\
0 & {\rm otherwise} 
\end{cases},
\end{align}
where $\gamma_i:=p_{i,k}^{(k-1)}/p_{k,k}^{(k-1)}$.

\item According to (\ref{eq:construct-affine}), we obtain
      $\Z^{(6)}=\Z^{(5)}+\I_{m+1}$ by using a network component.

\item By receiving $\P_{k-1}$ from a skip connection, we obtain
$\P_k:=\Z^{(6)}\P_{k-1}$, which is given by 
\begin{align}
\P_k=
\begin{bmatrix}
    p_{1,1} & \cdots & p_{1,k-1} & p_{1,k} & p_{1,k+1} & 
\cdots & p_{1,m+1}\\
               0 & \cdots & p_{2,k-1}^{(2)} & p_{2,k}^{(2)} & p_{2,k+1}^{(2)} & 
\cdots & p_{2,m+1}^{(2)}\\
\cdots & \cdots & \cdots & \cdots & \cdots & \cdots & \cdots\\
               0 & \cdots & p_{k-1,k-1}^{(k-1)} & p_{k-1,k}^{(k-1)} 
& p_{k-1,k+1}^{(k-1)} & \cdots & p_{k-1,m+1}^{(k-1)}\\
               0 & \cdots & 0 & p_{k,k}^{(k-1)} & p_{k,k+1}^{(k-1)} & 
\cdots & p_{k,m+1}^{(k-1)}\\
               0 & \cdots & 0 & 0 & p_{k+1,k+1}^{(k)} & 
\cdots & p_{k+1,m+1}^{(k)}\\
\cdots & \cdots \cdots & \cdots & \cdots & \cdots & \cdots\\
    0 & \cdots & 0 & 0 & p_{m,k+1}^{(k)} & \cdots & 
p_{m,m+1}^{(k)}\\
    0 & \cdots & 0 & 0 & 0 & \cdots & 0 
   \end{bmatrix},
\end{align}
where
\begin{align}
p_{i,j}^{(k)}&=p_{i,j}^{(k-1)}-\gamma_ip_{k,j}^{(k-1)}
\end{align}
for $i=k+1,\ldots,m$ and $j=k+1,\ldots,m+1$.
\end{itemize}

By repeating this procedure for $k=1,\ldots,m-1$, we can obtain
$\P_{m-1}$, in which $\P_{m-1}[1\:m,\:]$ is an upper triangular system
of linear equations and $\P_{m-1}[m+1,\:]=\O_{1,m+1}$.  We refer to this
module for the $k$-th step as $\FE_k$, by which we write
$\P_k=\FE_{k}(\P_{k-1})$.

\subsubsection{Module that solves upper triangular system}

The reminder is to construct a module that solves an upper triangular system
of linear equations. This executes backward substitution for $\P_{m-1}$.
Let $\Q$ be an $(m+1)\times(m+1)$, in which $\Q[1\:m,\:]$ is 
 an upper triangular system of linear
equations and the last row is $\O_{1,m+1}$; i.e.
\begin{align}
\Q= 
\begin{bmatrix}
    q_{1,1} & q_{1,2} & \cdots & q_{1,m-1} & q_{1,m} & q_{1,m+1}\\
    0 & q_{2,2} & \cdots & q_{2,m-1} & q_{2,m} & q_{2,m+1}\\
\cdots & \cdots & \cdots & \cdots & \cdots & \cdots\\
    0 & 0 & \cdots & q_{m-1,m-1} & q_{m-1,m} & q_{m-1,m+1}\\
    0 & 0 & \cdots & 0 & q_{m,m} & q_{m,m+1}\\
    0 & 0 & \cdots & 0 & 0 & 0
   \end{bmatrix}.
\end{align}
Here, we first show a processes for obtaining solutions of $m$-th and
$(m-1)$th variables (the last two solutions) as demonstrations and then
show a solving process in a general case.

\subsubsection{Computation of the $m$-th variable}

We solve the last equality. 
\begin{itemize}
 \item By a masking operation for $\Q$, we obtain an
       $(m+1)\times (m+1)$ matrix
$\Z^{(1)}:=\Z^{(1)}(\Q)$, in which
\begin{align}
\Z^{(1)}[i,j]=
\begin{cases}
q_{i,j} & i=j=m\\
0 & {\rm otherwise}
\end{cases}.
\end{align}

\item According to \ref{sec:invapprox}, by using a network component, we
       obtain an $(m+1)\times(m+1)$ matrix
       $\Z^{(2)}:=\Z^{(2)}(\Z^{(1)})$ whose $(i,j)$-entry is
\begin{align}
\Z^{(2)}[i,j]=
\begin{cases}
\invsqr(\Z^{(1)}[i,j])\simeq 1/{q_{m,m}^2} & i=j=m\\
\idrelu(\Z^{(1)}[i,j])=0 & {\rm otherwise}
\end{cases}.
\end{align}

\item By receiving $\Z^{(1)}$ from a skip connection, we obtain
       $\Z^{(3)}:=\Z^{(1)}\Z^{(2)}$, in which
\begin{align}
\Z^{(3)}[i,j]=
\begin{cases}
1/q_{m,m} & i=j=m\\
0 & {\rm otherwise}
\end{cases}.
\end{align}

\item According to (\ref{eq:construct-affine}), by using a network
       component,
we obtain an $(m+1)\times(m+1)$ matrix $\Z^{(4)}:=\Z^{(4)}(\Z^{(3)})$
       whose $(i,j)$-entry is
\begin{align}
\Z^{(4)}[i,j]=-\idrelu(\Z^{(3)}[i,j])+C_{i,j}=
\begin{cases}
-1/q_{m,m} & i=j=m\\
1 & i=j\neq m\\
0 & {\rm otherwise}
\end{cases},
\end{align}
where $C_{i,j}$ is set to $C_{i,j}=1$ if $i=j\neq m$ and $0$ otherwise. 

\item By receiving $\Q$ from a skip connection and using an anti-mask
       operation, we obtain
$\Q_m:=\oM_{m,m,m,m}^{m+1,m+1}\Z^{(4)}\Q$ which is given by
\begin{align}
\Q_m= 
\begin{bmatrix}
    q_{1,1} & q_{1,2} & \cdots & q_{1,m-1} & q_{1,m} & q_{1,m+1}\\
    0 & q_{2,2} & \cdots & q_{2,m-1} & q_{2,m} & q_{2,m+1}\\
\cdots & \cdots & \cdots & \cdots & \cdots & \cdots\\
    0 & 0 & \cdots & q_{m-1,m-1} & q_{m-1,m} & q_{m-1,m+1}\\
    0 & 0 & \cdots & 0 & 0 & \xi_m\\
    0 & 0 & \cdots & 0 & 0 & 0
   \end{bmatrix},
\end{align}
where $\xi_m:=q_{m,m+1}/q_{m,m}$ that is the solution to the $m$-th variable.
We need an anti-masking process here for keeping $\xi_m$ in $\Q_m$.
\end{itemize}

The module for computing a solution to $m$-th variable is denoted
$\BS1_m$ and we have $\Q_m=\BS1_m(\Q)$.

\subsubsection{Computation of the $(m-1)$-th variable}

As a demonstration, we next solve the $(m-1)$-th equality.  This process
is decomposed into two steps.  The first step is a process, in which we
insert $\xi_m$ to $m$-th variable and update $(m+1)$-th column of $\Q_m$.
The second step is the computation of a solution to $(m-1)$-th variable,
in which the updated $(m,m+1)$-entry is divided by $(m-1,m-1)$-entry.

\begin{itemize}

\item By a masking operation for $\Q_m$, we obtain an
      $(m+1)\times (m+1)$ matrix
$\Z^{(1)}:=\Z^{(1)}(\Q_m)$, in which
\begin{align}
\Z^{(1)}[i,j]=
\begin{cases}
\xi_m & i=m,~j=m+1\\
0 & {\rm otherwise}
\end{cases}.
\end{align}

\item According to (\ref{eq:construct-affine}), by using a network component,
we obtain an $(m+1)\times(m+1)$ matrix 
$\Z^{(2)}:=\Z^{(2)}(\Z^{(1)})$ whose $(i,j)$-entry is
\begin{align}
\Z^{(2)}[i,j]=-\idrelu(\Z^{(1)}[i,j])+C_{i,j}=
\begin{cases}
-\xi_m & i=m,~j=m+1\\
1 & i=j\\
0 & {\rm otherwise}
\end{cases},
\end{align}
where $C_{i,j}=\delta_{i,j}$. Indeed, this computes
$\Z^{(2)}[i,j]=-\Z^{(1)}[i,j]+\I_{m+1}$.

\item By receiving $\Q_m$ from a skip connection, we obtain
$\Z^{(3)}:=\Q_m\Z^{(2)}$ which is given by
\begin{align}
\Z^{(3)}= 
\begin{bmatrix}
    q_{1,1} & \cdots & q_{1,m-2} & q_{1,m-1} & q_{1,m} & \xi_{1,m}\\
\cdots & \cdots & \cdots & \cdots & \cdots & \cdots\\
    0 & \cdots & q_{m-2,m-2} & q_{m-2,m-1} & q_{m-2,m} & \xi_{m-2,m}\\
    0 & 0 & \cdots & q_{m-1,m-1} & q_{m-1,m} & \xi_{m-1,m}\\
    0 & 0 & \cdots & 0 & 0 & \xi_m\\
    0 & 0 & \cdots & 0 & 0 & 0
   \end{bmatrix},
\end{align}
where $\xi_{k,m}:=q_{k,m+1}-\xi_mq_{k,m}$ for $k=1,\ldots,m-1$.

\item Hereafter, we compute the solution to the $(k-1)$-th variable,
      which is almost the same as the computation of the $m$-th variable.
By a masking operation for $\Z^{(3)}$, we obtain an $(m+1)\times
      (m+1)$ matrix
$\Z^{(4)}:=\Z^{(4)}(\Z^{(3)})$, in which
\begin{align}
\Z^{(4)}[i,j]=
\begin{cases}
q_{i,j} & i=j=m-1\\
0 & {\rm otherwise}
\end{cases}.
\end{align}

\item According to \ref{sec:invapprox}, by using a network component, 
we obtain an $(m+1)\times(m+1)$ matrix
      $\Z^{(5)}:=\Z^{(5)}(\Z^{(4)})$, in which
\begin{align}
\Z^{(5)}[i,j]=
\begin{cases}
\invsqr(\Z^{(4)}[i,j])\simeq 1/{q_{i,j}^2} & i=j=m-1\\
\idrelu(\Z^{(4)}[i,j])=0 & {\rm otherwise}
\end{cases}.
\end{align}

\item By receiving $\Z^{(4)}$ from a skip connection, we have
       $\Z^{(6)}=\Z^{(5)}\Z^{(4)}$ whose $(i,j)$-entry is
\begin{align}
\Z^{(6)}[i,j]=
\begin{cases}
1/q_{i,j} & i=j=m-1\\
0 & {\rm otherwise}
\end{cases}.
\end{align}

\item According to (\ref{eq:construct-affine}), by using a network component,
we obtain an $(m+1)\times(m+1)$ matrix 
$\Z^{(7)}:=\Z^{(7)}(\Z^{(6)})$ whose $(i,j)$-entry is
\begin{align}
\Z^{(7)}[i,j]=\idrelu(\Z^{(6)}[i,j])+C_{i,j}=
\begin{cases}
1/q_{i,j} & i=j=m-1\\
1 & i=j\neq m-1\\
0 & {\rm otherwise}
\end{cases},
\end{align}
where $C_{i,j}$ is set to $C_{i,j}=1$ if $i=j\neq m-1$ and $0$ otherwise. 

\item By receiving $\Z^{(3)}$ from a skip connection and using an
      anti-mask operation, we obtain
$\Q_{m-1}:=\oM_{m-1,m-1,m-1,m-1}^{m+1,m+1}\Z^{(7)}\Z^{(3)}$, 
which is given by
\begin{align}
\Q_{m-1}= 
\begin{bmatrix}
    q_{1,1} & \cdots & q_{1,m-2}& q_{1,m-1} & q_{1,m} & \xi_{1,m}\\
\cdots & \cdots & \cdots & \cdots & \cdots & \cdots\\
    0 & \cdots & q_{m-2,m-2} & q_{m-2,m-1} & q_{m-2,m} & \xi_{m-2,m}\\
    0 & \cdots & 0 & 0 & q_{m-1,m}/q_{m-1,m-1} & \xi_{m-1}\\
    0 & \cdots & 0 & 0 & 0 & \xi_m\\
    0 & \cdots & 0 & 0 & 0 & 0
   \end{bmatrix},
\end{align}
where 
\begin{align}
\xi_{m-1}:=\xi_{m-1,m}/q_{m-1,m-1}=
\frac{1}{q_{m-1,m-1}}\left(
q_{k,m+1}-\xi_mq_{k,m}\right),
\end{align}
which is the solution to the $(m-1)$-th variable.
\end{itemize}

\subsubsection{Implementation of backward substitution}
\label{sec:BS}

Along this construction for the $(m-1)$-th variable, a solution to the $(m-s-1)$-th
variable for $s=0,1,\ldots,m-2$ is obtained by the following steps.
We set $t=m-s$ for simplifying the expressions. We have
\begin{align}
\Q_t:= 
\begin{bmatrix}
q_{1,1} &  \cdots & q_{1,t-2} & q_{1,t-1} & q_{1,t} & \cdots  & \xi_{1,t}\\
\cdots &  \cdots & \cdots & \cdots  & \cdots & \cdots & \cdots\\
        0 & \cdots & q_{t-2,t-2}& q_{t-2,t-1} & q_{t-2,t} & \cdots & \xi_{t-2,t}\\
        0 & \cdots & 0 & q_{t-1,t-1} & q_{t-1,t} & \cdots & \xi_{t-1,t}\\
        0 & \cdots & 0 & 0 & 0 & \cdots &  \xi_t\\
\cdots & \cdots & \cdots  & \cdots & \cdots & \cdots\\
        0 & \cdots & 0 &       0   & \cdots & \cdots & \xi_m\\
        0 & \cdots & 0 &       0   & \cdots & \cdots & 0
\end{bmatrix},
\end{align}
where
\begin{align}
\xi_{k,t}=
q_{k,m+1}-\sum_{j=1}^{s}\xi_{m-j+1}q_{k,m-j+1}.
\end{align}
and $\xi_k$ is a solution to the $k$-th variable. 
Under this setting, we find a solution to $(t-1)$-th variable.

\begin{itemize}

\item By a masking operation for $\Q_t$, we obtain an $(m+1)\times (m+1)$ matrix
$\Z^{(1)}:=\Z^{(1)}(\Q_t)$, in which
\begin{align}
\Z^{(1)}[i,j]=
\begin{cases}
\xi_t & i=t,~j=m+1\\
0 & {\rm otherwise}
\end{cases}.
\end{align}

\item According to (\ref{eq:construct-affine}), by using a network
      component,
we obtain an $(m+1)\times(m+1)$ matrix 
$\Z^{(2)}:=\Z^{(2)}(\Z^{(1)})=-\Z^{(1)}[i,j]+\I_{m+1}$.

\item By receiving $\Q_t$ from a skip connection, we obtain
$\Z^{(3)}=\Q_t\Z^{(2)}$, which is given by
\begin{align}
\Z^{(3)}= 
\begin{bmatrix}
q_{1,1} &  \cdots & q_{1,t-2} & q_{1,t-1} & q_{1,t} & \cdots  & \xi_{1,t-1}\\
\cdots &  \cdots & \cdots & \cdots  & \cdots & \cdots & \cdots\\
        0 & \cdots & q_{t-2,t-2}& q_{t-2,t-1} & q_{t-2,t} & \cdots & \xi_{t-2,t-1}\\
        0 & \cdots & 0 & q_{t-1,t-1} & q_{t-1,t} & \cdots & \xi_{t-1,t-1}\\
        0 & \cdots & 0 & 0 & 0 & \cdots &  \xi_t\\
\cdots & \cdots & \cdots  & \cdots & \cdots & \cdots\\
        0 & \cdots & 0 &       0   & \cdots & \cdots & \xi_m\\
        0 & \cdots & 0 &       0   & \cdots & \cdots & 0
\end{bmatrix},
\end{align}
where $\xi_{k,t-1}:=\xi_{k,t}-\xi_tq_{k,t}$.

\item By using a masking operation for $\Z^{(3)}$, we obtain an
      $(m+1)\times (m+1)$ matrix
$\Z^{(4)}:=\Z^{(4)}(\Z^{(3)})$, in which
\begin{align}
\Z^{(4)}[i,j]=
\begin{cases}
q_{i,j} & i=j=t-1\\
0 & {\rm otherwise}
\end{cases}.
\end{align}

\item We obtain an $(m+1)\times(m+1)$ matrix
      $\Z^{(5)}:=\Z^{(5)}(\Z^{(4)})$, in which
\begin{align}
\Z^{(5)}[i,j]=
\begin{cases}
\invsqr(\Z^{(4)}[i,j])\simeq 1/{q_{i,j}^2} & i=j=t-1\\
\idrelu(\Z^{(4)}[i,j])=0 & {\rm otherwise}
\end{cases}.
\end{align}

\item By receiving $\Z^{(4)}$ from a skip connection, we have
       $\Z^{(6)}=\Z^{(5)}\Z^{(4)}$ whose $(i,j)$-entry is
\begin{align}
\Z^{(6)}[i,j]=
\begin{cases}
1/q_{i,j} & i=j=t-1\\
0 & {\rm otherwise}
\end{cases}.
\end{align}

\item According to (\ref{eq:construct-affine}), by using a network
      component,
we obtain an $(m+1)\times(m+1)$ matrix 
$\Z^{(7)}:=\Z^{(7)}(\Z^{(6)})$ whose $(i,j)$-entry is
\begin{align}
\Z^{(7)}[i,j]=\idrelu(\Z^{(6)}[i,j])+C_{i,j}=
\begin{cases}
1/q_{i,j} & i=j=t-1\\
1 & i=j\neq t-1\\
0 & {\rm otherwise}
\end{cases},
\end{align}
where $C_{i,j}$ is set to $C_{i,j}=1$ if $i=j\neq t-1$ and $0$ otherwise. 

\item By receiving $\Z^{(3)}$ from a skip connection and using an
      anti-mask operation, we obtain
$\Q_{t-1}:=\oM_{t-1,t-1,t-1,t-1}^{m+1,m+1}\Z^{(6)}\Z^{(3)}$, which is given by
\begin{align}
\Q_{t-1}= 
\begin{bmatrix}
q_{1,1} &  \cdots & q_{1,t-2} & q_{1,t-1} & q_{1,t} & \cdots  & \xi_{1,t-1}\\
\cdots &  \cdots & \cdots & \cdots  & \cdots & \cdots & \cdots\\
        0 & \cdots & q_{t-2,t-2}& q_{t-2,t-1} & q_{t-2,t} & \cdots & \xi_{t-2,t-1}\\
        0 & \cdots & 0 & 0 & q_{t-1,t}/q_{t-1,t-1} & \cdots & \xi_{t-1}\\
        0 & \cdots & 0 & 0 & \cdots & \cdots &  \xi_t\\
\cdots & \cdots & \cdots  & \cdots & \cdots & \cdots\\
        0 & \cdots & 0 &       0   & \cdots & \cdots & \xi_m
\end{bmatrix},
\end{align}
where 
\begin{align}
\xi_{t-1}:=\xi_{t-1,t-1}/q_{t-1,t-1}
\end{align}
which is a solution to the $(t-1)$-th variable.

\end{itemize}

The module for computing a solution to $(m-s-1)$-th variable for
$s=0,1,\ldots,m-2$ is denoted by $\BS1_{m-s-1}$ and we have
$\Q_{m-s-1}=\BS1_{m-s-1}(\Q_{m-s})$.  Note that $\BS1_m$ is defined in
case of demonstrating a solver for the $m$-th variable. 
As a summary, let $\Q$ be an
$(m+1)\times (m+1)$ matrix, in which 
$\Q[1\:m,\:]$ is an upper triangular system of $m$ linear equations to
solved and $\Q[m+1,\:]=\O_{1,m+1}$.
We repeatedly compute $\Q_{k-1}=\BS1_k(\Q_k)$ for $k=m,\ldots,2$.
As a result, $\Q_1[1\:m,m+1]$ is a vector of solutions.
We refer to this module by $\BS_m$; i.e. $\Q_1=\BS_m(\Q_m)$.

\subsubsection{Remarks}

We have several remarks in the appendix.
\begin{itemize}

\item \cite{Akyurek2022} has also shown an implementation for solving
a system of linear equations, in which they implemented the solver using
      self-attention components while we do not consider a network
      structure.
The important point is the computational aspect of the implementations.
In implementing a computational algorithm, 
we need arithmetic operations, which are addition, subtraction,
      multiplication and division. The difficulty arises in the
      implementation multiplication and division operations.
In \cite{Akyurek2022}, multiplication was implemented by a network component
      using ReLU activation function and division by a batch normalization.
In our implementation, we use matrix multiplication for multiplication and 
a network component using a ReLU activation function for division. Note that 
the network component is also needed if the implementation requires the
      computation of nonlinear functions; e.g. loss functions, see
      \cite{Bai2023,Hataya2024}.

 \item The above implementation uses a masking operation, a
      multiplication type skip connection, a network component and
      matrix multiplication. The masking operation is implemented by a network
      component. However, we emphasize that the network component and
      the weighted sum operation cannot realize multiplication of two
      input matrices. The matrix multiplication plays a key role in
      transformers. And, the network component and the weighted sum
      operation are used before transformers; e.g. convolutional
      neural networks. These facts may imply that matrix multiplication 
      may be important for in-context learning.

\item We have implemented a Gaussian elimination procedure to solve a
      system of linear equations using standard components used in the
      construction of neural networks.  Unfortunately, although this
      implementation may be straightforward, it can be difficult to
      embed in a simple layered form that consists of sequentially
      connecting a common module. In addition, for example, if we consider
      in-context learning of ridge regression then we need a module that
      transforms the data into an input for the Gaussian elimination module.
      This may require a different form from the Gaussian elimination
      module.  Therefore, as used in \cite{Akyurek2022,Bai2023},
      in-context learning of the gradient descent type seems valid since
      it can be implemented by sequentially connecting a simple module
      that performs one-step update procedure.

\end{itemize}

\end{document}